\RequirePackage{fix-cm}
\documentclass{IEEEoj}
\usepackage{amsmath,amsfonts}
\usepackage{algorithmic}
\usepackage{algorithm}
\usepackage{array}
\usepackage[caption=false,font=footnotesize,labelfont=rm,textfont=rm]{subfig}
\usepackage{textcomp}
\usepackage{stfloats}
\usepackage{url}
\usepackage{verbatim}
\usepackage{graphicx}
\usepackage{xcolor}
\usepackage{cite}
\usepackage{bm}
\usepackage{makecell}
\usepackage{microtype}  % 优化字符间距
\usepackage{xurl} % 智能处理URL断行
\usepackage{booktabs} % 美化表格
\usepackage{multirow} % 表格多行合并
\usepackage{hyperref}    % 加载超链接宏包
\usepackage[capitalize]{cleveref} % 智能引用并首字母大写
\usepackage{amsthm}      % 加载定理环境宏包
\usepackage{orcidlink}
\usepackage{tabularx} % 用于创建自适应宽度表格
\newtheoremstyle{italicTheorem}
  {\topsep}   % space above
  {\topsep}   % space below
  {\normalfont}  % body font
  {\parindent}       % indent amount
  {\itshape}  % theorem head font
  {:}          % punctuation after theorem head
  {.5em}       % space after theorem head
  {\thmname{#1}\thmnumber{ #2}\thmnote{ (#3)}} % 括号内容也设为斜体

\theoremstyle{italicTheorem}
\newtheorem{theorem}{Theorem}
\newtheorem{lemma}{Lemma} 

\makeatletter
\renewenvironment{proof}[1][\proofname]{\par
  \pushQED{\qed}%
  \normalfont
  \topsep6\p@\@plus6\p@\relax
  \trivlist
  \item[\hskip\labelsep\itshape #1]\ignorespaces
}{%
  \popQED\endtrivlist\@endpefalse
}
\makeatother
\renewcommand{\proofname}{Proof:}

\AtBeginDocument{\definecolor{ojcolor}{cmyk}{0.93,0.59,0.15,0.02}}
\def\OJlogo{\vspace{-4pt}\hskip-4pt\includegraphics[height=18pt]{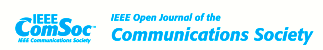}}
\graphicspath{{figures/}}
\hypersetup{hidelinks,pdftitle={Joint Interference Detection and Identification via Adversarial Multi-task Learning},pdfauthor={Haitao Xu, Li Hu, Boxiang He, Shilian Wang}}

\usepackage{ojcoms-submission}
\usepackage{needspace}

\newif\ifRonechanges
\Ronechangesfalse
\newcommand{\revised}[1]{{\ifRonechanges\color{red}\else\color{black}\fi#1}}
\newenvironment{Ronechanged}{\begingroup\ifRonechanges\color{red}\else\color{black}\fi}{\endgroup}

\begin{document}
\bstctlcite{IEEEexample:BSTcontrol}
\receiveddate{}
\reviseddate{}
\accepteddate{}
\publisheddate{}
\currentdate{}
\doiinfo{}
\title{Joint Interference Detection and Identification\\via Adversarial Multi-task Learning}
\author{Haitao Xu\IEEEauthorrefmark{1}\orcidlink{0009-0009-7667-0522},
Li Hu\IEEEauthorrefmark{1},
Boxiang He\IEEEauthorrefmark{1}\orcidlink{0000-0002-9235-1144}, and
Shilian Wang\IEEEauthorrefmark{1}\orcidlink{0000-0003-4132-8750}}
\affil{College of Electronic Science and Technology, National University of Defense Technology, Changsha 410003, P. R. China (e-mail: HaitaoXu1997@outlook.com; huli09@nudt.edu.cn; boxianghe1@bjtu.edu.cn; wangsl@nudt.edu.cn).}
% TODO: Set the corresponding author after author confirmation.
\corresp{}
% No funding statement was present in the source manuscript.
\authornote{}
\markboth{Joint Interference Detection and Identification via Adversarial Multi-task Learning}{Xu et al.}
\begin{abstract}
Precise interference detection and identification are crucial for enhancing the survivability of communication systems in non-cooperative wireless environments. 
While deep learning (DL) has advanced this field, existing single-task learning (STL) approaches neglect inherent task correlations. 
Furthermore, emerging multi-task learning (MTL) methods often lack a theoretical foundation for quantifying and modeling task relationships. 
To bridge this gap, we establish a theoretically grounded MTL framework for joint interference detection, modulation identification, and interference identification. 
First, we derive an upper bound for the weighted expected loss in MTL frameworks. 
This bound explicitly connects MTL performance to task similarity, quantified by the Wasserstein distance and learnable task relation coefficients. 
Guided by this theory, we present the adversarial multi-task interference detection and identification network (AMTIDIN), 
which integrates adversarial training to minimize distributional discrepancies across tasks and uses adaptive coefficients to model task correlations dynamically. 
Crucially, we conducted a quantitative analysis of task similarity to reveal intrinsic task relationships, specifically that modulation identification and interference identification share a substantial feature overlap distinct from interference detection.
\revised{Experiments demonstrate that AMTIDIN outperforms its independently trained single-task counterparts and MTL baselines under the evaluated conditions of limited training data, short signal lengths, and low signal-to-noise ratios (SNRs).}
\end{abstract}


\maketitle

\section{Introduction}
\label{sec:introduction}

% 背景介绍
\IEEEPARstart{O}{wing} to the inherent openness of wireless channels in non-cooperative wireless communication environments, ranging from military communication networks to unlicensed spectrum sharing scenarios, the electromagnetic spectrum is becoming increasingly congested and contested \cite{He2025Towards,Qiao2026PASS}.
Consequently, communication systems are highly susceptible to a diverse array of threats, including unintentional interference and malicious jamming attacks, which severely compromise the reliability and survivability of wireless links \cite{Wang2020Dynamic, Pirayesh2022Jamming}.
\revised{Furthermore, covert transmissions introduce additional challenges for non-cooperative spectrum monitoring by seeking to conceal their presence~\cite{Chen2023CovertSurvey}.}
To mitigate \revised{interference threats}, advanced anti-interference strategies, such as interference avoidance, suppression, and elimination, have been widely developed. 
However, the effective implementation of these countermeasures fundamentally relies on precise interference detection and identification to capture and characterize the specific interference patterns beforehand.

% 传统方法及其局限性
Traditional methodologies are broadly categorized into likelihood-based (LB) and feature-based (FB) paradigms.
LB approaches typically formulate the problem within a statistical hypothesis testing framework.
For instance, \revised{Zhao et al.}~\cite{Zhao2017Discrimination} developed a discrimination scheme based on the \revised{generalized likelihood ratio test} under the Neyman-Pearson criterion, using the distinct spatial subspace properties of signal vectors to distinguish deception jamming from radar targets. 
However, these methods are often constrained by prohibitive computational complexity and the necessity for precise prior knowledge, such as \revised{channel state information} or specific jamming parameters.
Conversely, FB approaches rely on extracting domain-expert features, including amplitude, phase, high-order cumulants \cite{Wang2018Leveraging}, and cyclic spectrum \cite{Yan2020Robust}.
While multi-feature fusion can enhance identification accuracy \cite{Zuo2021Recognition, Wang2017Interference}, this paradigm remains fundamentally limited by its dependence on laborious manual feature engineering, which significantly increases system complexity and restricts adaptability to dynamic environments.

% 单任务深度学习方法及其局限性
Deep learning (DL) has emerged as a potent tool for interference detection and identification, leveraging its ability to autonomously extract high-dimensional representations from raw data.
Existing research has extensively explored \mbox{DL-based} single-task learning (STL) paradigms.
For interference detection, approaches typically employ architectures such as convolutional neural networks (CNNs), \revised{convolutional long short-term memory networks}, or object detection frameworks to identify signal presence or localize specific waveforms like frequency-hopping signals \cite{Gao2019Deep, Wang2023ConvLSTMbased, Prasad2020Downscaled}.
In the realm of modulation identification, methodologies have progressed from supervised models using \revised{long short-term memory networks} \cite{Rajendran2018Deep} or hybrid architectures \cite{Qiu2024DeepSIG} to advanced semi-supervised and unsupervised contrastive learning paradigms designed to tackle label scarcity \cite{Kong2023Transformerbased, Li2025Multirepresentation}.
Regarding interference identification, research has advanced from pioneering CNN-based classifiers \cite{Schmidt2017Wireless} to efficient lightweight architectures \cite{Wang2025Intelligent} and sophisticated multi-domain \cite{Wang2022Multidomain} or multi-modal \cite{Wang2024Wireless} feature fusion frameworks.
\revised{However, learning these tasks independently limits the use of shared signal characteristics and cross-task relationships.}

% 多任务学习的潜力及其局限性
\begin{Ronechanged}
Multi-task learning (MTL) exploits task relationships to improve signal recognition through information sharing.
Zhang et al.~\cite{Zhang2023JDMRnet} jointly addressed pulse detection and modulation recognition for low-probability-of-intercept radar signals.
Xing et al.~\cite{Xing2024Joint} developed a joint framework for signal parameter estimation and modulation classification in multi-signal scenarios.
Gong et al.~\cite{Gong2022Multitask} combined wireless signal classification with auxiliary modulation classification to improve open-set identification in the industrial, scientific, and medical band.
More recently, Zhao et al.~\cite{Zhao2026Robust} integrated modulation classification and signal type recognition through multi-scale feature decomposition and fusion.
These studies demonstrate the benefits of joint processing.
However, whether and to what extent information can be shared across tasks remains insufficiently characterized.
\end{Ronechanged}

% 机器学习社区中任务相似性的测量方法
\begin{Ronechanged}
Characterizing task relationships requires measuring task similarity, as more similar tasks are generally better suited to information sharing.
Task relation coefficients weight the contribution of each task to the learning objective of another task \cite{Murugesan2016Adaptive, Pentina2017Multitask}, while distribution metrics, such as H-divergence \cite{Ben-David2010Theory} and Wasserstein distance \cite{Redko2017Theoretical, Shen2018Wasserstein}, quantify discrepancies between task distributions.
These complementary measures provide a basis for quantifying task similarity and guiding information sharing in joint interference detection and identification.
\end{Ronechanged}

% 论文贡献
\begin{Ronechanged}
Interference detection, modulation identification, and interference identification characterize different aspects of the received signal, yet share underlying signal characteristics.
To exploit these relationships, we establish a unified theoretical framework that relates the weighted expected loss to task distribution distances and task relation coefficients.
Guided by this framework, we develop the adversarial multi-task interference detection and identification network (AMTIDIN), integrating adversarial feature alignment with adaptive cross-task loss weighting.
\end{Ronechanged}
The main contributions are summarized as follows:
\begin{itemize}
    \item 
    \revised{We derive an upper bound on the weighted expected loss for joint interference detection, modulation identification, and interference identification.
    The bound relates the expected loss to Wasserstein distances and task relation coefficients, providing theoretical guidance for feature alignment and cross-task loss weighting.}
    \item 
    \revised{Guided by this analysis, we design AMTIDIN to jointly perform the three tasks.
    Adversarial training reduces discrepancies between task feature distributions, while learnable task relation coefficients adaptively weight the contribution of each task to the training of the others.}
    \item 
    \revised{We quantify task similarity using Wasserstein distances and learned task relation coefficients.
    Modulation identification and interference identification exhibit smaller distributional discrepancies and larger task relation coefficients than either does with interference detection, indicating greater knowledge and information sharing between these two tasks during MTL.}
    \item 
    \revised{Experiments demonstrate improved classification accuracy over independently trained single-task counterparts and MTL baselines using the same architectures for corresponding network modules, particularly with limited training samples, short signals, and low signal-to-noise ratios.
    These controlled comparisons assess the contributions of joint learning, adaptive task weighting, and adversarial alignment.}
\end{itemize}

The remainder of this paper is organized as follows: 
Section \ref{sec:system_model} presents the system model and formulates the joint interference detection and identification problem. 
Section \ref{sec:theoretical_analysis} provides a comprehensive theoretical analysis, where we derive an upper bound for the MTL framework and quantify task similarity. 
Building on these theoretical insights, Section \ref{sec:proposed_framework} details the proposed AMTIDIN architecture and its algorithmic implementation.
Section \ref{sec:experimental_results} presents the experimental setup and a comprehensive performance analysis. 
Finally, Section \ref{sec:conclusion} concludes the paper and outlines future research directions.

\revised{\emph{Notation:} Bold lowercase and uppercase letters denote vectors and matrices, respectively. $\mathbb{R}$ and $\mathbb{C}$ denote the sets of real and complex numbers. $(\cdot)^{\mathsf{T}}$ denotes transposition, $|\cdot|$ denotes absolute value or modulus, and $\|\cdot\|_1$ and $\|\cdot\|_2$ denote the $\ell_1$ and Euclidean norms, respectively. $\log$ denotes the natural logarithm. $\mathbb{E}[\cdot]$ denotes expectation, and $\mathbb{I}(\cdot)$ is the indicator function, which equals 1 if its argument is true and 0 otherwise. $\mathcal{CN}(\bm{\mu},\bm{\Sigma})$ denotes a circularly symmetric complex Gaussian distribution with mean $\bm{\mu}$ and covariance matrix $\bm{\Sigma}$. $\bm{I}_n$ denotes the $n\times n$ identity matrix, and $\mathcal{I}_N=\{1,2,\ldots,N\}$ denotes an index set.}

\section{System Model and Problem Formulation}
\label{sec:system_model}
This section presents the system model and formulates joint interference detection, modulation identification, 
and interference identification as an MTL problem.

\subsection{System Model}
\revised{As illustrated in Fig.~\ref{fig-SystemModel}, the intelligent interference monitoring receiver monitors a spectrum environment containing diverse interference and jamming links.
The receiver operates during a spectrum monitoring phase, such as a collaborative quiet period or an unoccupied frequency band, where legitimate communication is absent.}
\revised{The received complex baseband signal at the intelligent interference monitoring receiver can be modeled as}
\begin{equation}
\bm{r} = g \bm{s}_{m,i} + \bm{n},
\end{equation}
where $g \in \mathbb{C}$ is the fading channel; 
$\bm{n} \sim \mathcal{CN}\left(\bm{0}, \sigma^2 \bm{I}_N\right)$ represents the additive white Gaussian noise (AWGN);
$\bm{s}_{m,i}$ is the transmitted interference signal characterized by $m$-th modulation type and $i$-th interference type.
After preprocessing, the received signal matrix can be further represented as
\begin{equation}
    \bm{X} = \begin{bmatrix}
        \bm{r}_I^\mathsf{T} \\
        \bm{r}_Q^\mathsf{T}
    \end{bmatrix} = \begin{bmatrix}
        r_I[1] & r_I[2] & \cdots & r_I[N] \\
        r_Q[1] & r_Q[2] & \cdots & r_Q[N]
    \end{bmatrix},
    \label{eq:iq_matrix}
\end{equation}
where $\bm{r}_I$ and $\bm{r}_Q$ are the in-phase and quadrature components, respectively; $N$ is the signal length.

\begin{figure}[!ht]
\centering
\includegraphics[width=0.99\linewidth]{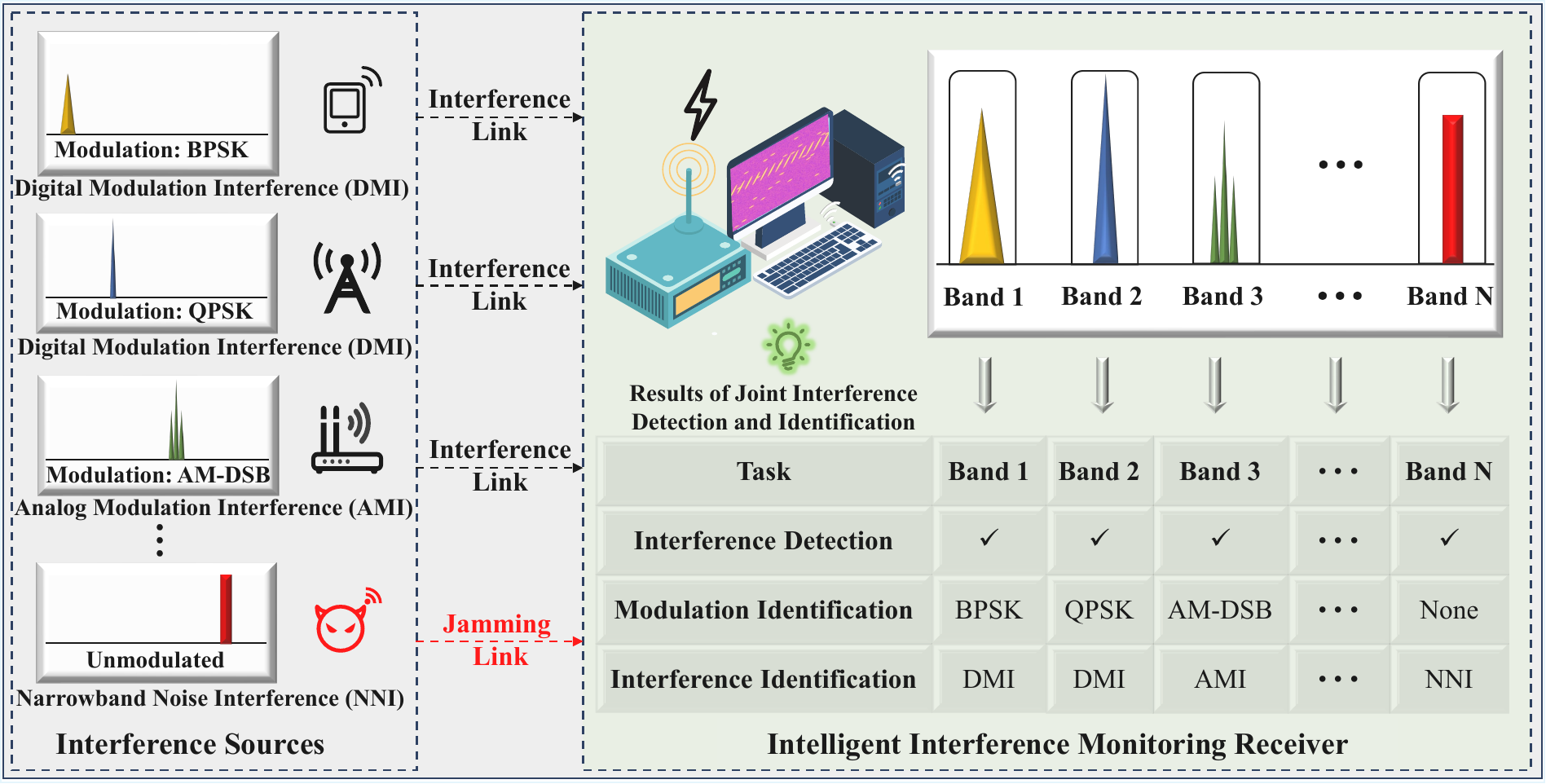}
\caption{System model for joint interference detection and identification within a non-cooperative electromagnetic spectrum environment.}
\label{fig-SystemModel}
\end{figure}

\subsection{Problem Formulation}
Our goal is to design a multi-task model $\mathcal{G}$ for an intelligent interference monitoring receiver to achieve interference detection, modulation identification, and interference identification simultaneously, which can be mathematically expressed by
\begin{equation}
\mathcal{G}: \bm{X} \rightarrow \{\hat{y}^\text{ID}, \hat{y}^\text{MI}, \hat{y}^\text{II}\},
\end{equation}
where $\hat{y}^\text{ID} \in \{1, 2\}$ is the predicted label indicating the presence or absence of interference; 
$\hat{y}^\text{MI} \in \revised{\mathcal{I}_{c_\mathrm{MI}}}$ denotes the predicted modulation type of the interference signal;
$\hat{y}^\text{II} \in \revised{\mathcal{I}_{c_\mathrm{II}}}$ is the predicted interference type. \revised{Here, $c_\mathrm{ID}=2$, $c_\mathrm{MI}$, and $c_\mathrm{II}$ denote the numbers of classes for the three tasks, respectively.}

To obtain the optimal model $\mathcal{G}^*$, the intuitive approach is to maximize the joint posterior probability of the corresponding labels given the input $\bm{X}$, which can be formulated as
\begin{equation}
\label {equation-joint optimization}
\mathcal{G}^* = \arg\max_{\mathcal{G}} P\left(y^\text{ID}, y^\text{MI}, y^\text{II} | \bm{X};\mathcal{G}\right),   
\end{equation}
where $y^{\text{ID}}$, $y^{\text{MI}}$, and $y^{\text{II}}$ are the corresponding ground-truth labels for each task.

\revised{However, directly modeling the joint posterior distribution is challenging because of the complex relationships among tasks and the high dimensionality of the input signal matrix $\bm X$. To address the joint recognition problem in~\eqref{equation-joint optimization}, we adopt an MTL framework that jointly learns the task hypotheses $\{h_\text{ID},h_\text{MI},h_\text{II}\}\subset\mathcal H$, where $\mathcal H$ is the hypothesis space. The learning objective is to minimize the weighted expected loss across the three tasks}
\begin{align}
\label{equation-MTL-optimization}
\mathcal{G}^* & = \{h_\text{ID}^*, h_\text{MI}^*, h_\text{II}^*\} \notag \\
&= \mathop{\arg\min}_{\{h_\text{ID}, h_\text{MI}, h_\text{II}\} \subset \mathcal{H}} \lambda_\text{ID} \mathcal{L}_\text{ID}\left(h_\text{ID}\right) \notag \\
&\quad + \lambda_\text{MI} \mathcal{L}_\text{MI}\left(h_\text{MI}\right) + \lambda_\text{II} \mathcal{L}_\text{II}\left(h_\text{II}\right),
\end{align}
where $\lambda_\text{ID}$, $\lambda_\text{MI}$, and $\lambda_\text{II}$ are the task-specific weights indicating the importance of each task in the overall learning objective.
Here, the expected loss for interference detection is defined as
\begin{equation}
% R1 CHANGE: 01-risk-equation
\ifRonechanges\color{red}\else\color{black}\fi
\mathcal{L}_\text{ID}(h_\text{ID}) = \mathbb{E}_{(\bm{X},y)\sim P_\text{ID}}
\bigl[\ell_\text{ID}(h_\text{ID},\bm{X},y)\bigr],
\end{equation}
% R1 CHANGE: 02-joint-marginal-notation
\revised{where $P_i$ denotes the joint distribution of inputs and labels for task $i$, and $\mathcal{D}^i$ denotes its marginal distribution over inputs. The bounded multiclass loss $\ell_i$ is defined in Section~\ref{sec:theoretical_analysis}.} Modulation identification and interference identification can also be defined similarly, but for brevity, we omit the details here.

Classical MTL approaches typically approximate the optimization problem \eqref{equation-MTL-optimization} by minimizing the weighted empirical loss over the training dataset.
However, this approach leaves the mathematical relationships among tasks largely unexplained.
To solve this, we establish a rigorous theoretical framework that explicitly quantifies and models task relationships to enhance MTL performance.

\section{Theoretical Analysis}
\label{sec:theoretical_analysis}
In this section, we first introduce the Wasserstein-1 distance and the task relation coefficients as metrics to quantify task similarity. We next derive an upper bound for the weighted expected loss across tasks, providing a theoretical foundation for our adversarial MTL framework.

\subsection{Metrics for Task Similarity}
To measure task similarity from a distributional perspective, we use the Wasserstein-1 distance \cite{Shen2018Wasserstein}, defined as
\begin{equation}
\label{equation-Wasserstein distance}
\!W_1\left(\mathcal{D}^t, \mathcal{D}^i\right) \! = \!\!\! \sup_{f: \|f\|_{\text{Lip}} \leq 1} \!\!\!\! \mathbb{E}_{\bm{X}^t \sim \mathcal{D}^t} f\left(\bm{X}^t\right) - \mathbb{E}_{\bm{X}^i \sim \mathcal{D}^i} f\left(\bm{X}^i\right), 
\end{equation}
where $\|f\|_{\text{Lip}} = \sup_{\bm{X}^t \neq \bm{X}^i} \frac{|f\left(\bm{X}^t\right) - f\left(\bm{X}^i\right)|}{\|\bm{X}^t - \bm{X}^i\|_2} \leq 1$ means $f$ is a 1-Lipschitz function;
$\bm{X}^t$ and $\bm{X}^i$ are samples drawn from the underlying distributions $\mathcal{D}^t$ and $\mathcal{D}^i$ of tasks $t$ and $i$, respectively.
From the perspective of task relationships, we also introduce the learnable task relation coefficients to further quantify task similarity \cite{Murugesan2016Adaptive}.
\revised{For $T$ tasks, we define a task relation coefficient vector $\bm{\alpha}_t\in\Delta_T$ for each task $t$, where $\Delta_T=\{\bm a\in\mathbb R^T:a_i\geq0,\sum_{i=1}^T a_i=1\}$.}
Each element $\alpha_{t,i}$ in $\bm{\alpha}_t$ represents the contribution of task $i$ to task $t$, with larger values indicating a stronger relationship between the two tasks.

\subsection{Upper Bound for Weighted Expected Loss}

% R1 CHANGE: 03-multiclass-theorem
\begin{Ronechanged}
For task $i$ with $c_i$ classes, a hypothesis $h\in\mathcal H$ contains task-specific probability outputs $h^{(i)}:\mathcal X\to\Delta_{c_i}$, where $\mathcal X$ is the input space and $\Delta_{c_i}=\{\bm p\in\mathbb R^{c_i}:p_c\geq0,\sum_{c=1}^{c_i}p_c=1\}$ is the class-probability simplex.
For branch $t$, $h_t^{(i)}$ is its classifier for task $i$, matching the nine-head architecture in Section~\ref{sec:proposed_framework}.
With one-hot label $\bm e_y$ and predicted probability $h_y^{(i)}(x)$ for class $y$, we use the bounded multiclass loss
\begin{equation}
\label{eq:multiclass_loss}
\ell_i(h,x,y)=\tfrac12\|h^{(i)}(x)-\bm e_y\|_1
=1-h_y^{(i)}(x)\in[0,1].
\end{equation}
Thus, the loss is scalar-valued, while each prediction is a $c_i$-dimensional probability vector.
Let $\mathcal F_i=\{(x,y)\mapsto\ell_i(h,x,y):h\in\mathcal H\}$ denote the scalar loss class for task $i$, and let $\mathfrak R_{m_i}(\mathcal F_i)$ denote its expected Rademacher complexity~\cite{Bartlett2002Rademacher}, defined in Appendix~A.
For any two hypotheses $h,h'\in\mathcal H$, let $d_i(\bm p,\bm q)=\tfrac12\|\bm p-\bm q\|_1$ and define their prediction disagreement on task $i$ as $q_i^{h,h'}(x)=d_i(h^{(i)}(x),h'^{(i)}(x))$.
To compare tasks with different label spaces, define
\begin{equation}
\label{eq:head_discrepancy}
\eta_{i,t}=\sup_{h,h'\in\mathcal H}\sup_{x\in\mathcal X}
|q_i^{h,h'}(x)-q_t^{h,h'}(x)|.
\end{equation}
The term $\eta_{i,t}$ bounds differences in scalar prediction disagreement across task-specific outputs, without requiring correspondence between their class labels.

\begin{theorem}
\label{theorem-Upper bound}
Let each task provide $m_i$ i.i.d. labeled observations, with $m=\sum_i m_i$ and $\beta_i=m_i/m>0$.
Let $\lambda_t\geq0$ be fixed task weights satisfying $\sum_{t=1}^{T}\lambda_t=1$.
Assume every output $h^{(i)}$ is $K$-Lipschitz under $d_i$ on the common input space $(\mathcal X,d_{\mathcal X})$, where $d_{\mathcal X}$ is the input metric.
Assume the Wasserstein concentration conditions in Lemma~\ref{lem:wasserstein_concentration} of Appendix~\ref{appendix:proof_of_theorem_1}.
For $\delta\in(0,1)$, with probability at least $1-\delta$, simultaneously for all $h_t\in\mathcal H$ and $\bm\alpha_t\in\Delta_T$,
\begin{align}
\label{equation-Upper bound}
\underbrace{\sum_{t=1}^T\lambda_t\mathcal L_t(h_t)}_{\substack{\text{Weighted}\\\text{expected loss}}}
&\leq\underbrace{\sum_{t=1}^T\lambda_t\hat{\mathcal L}_{\bm\alpha_t}(h_t)}_{\substack{\text{Weighted}\\\text{empirical loss}}}
+\underbrace{C_1\sum_{t=1}^T\lambda_t\sqrt{\sum_{i=1}^T\frac{\alpha_{t,i}^2}{\beta_i}}}_{\substack{\text{Coefficient-dependent}\\\text{complexity penalty}}}\notag\\
&\quad+\underbrace{2K\sum_{t=1}^T\lambda_t\sum_{i=1}^T\alpha_{t,i}
W_1(\hat{\mathcal D}^i,\hat{\mathcal D}^t)}_{\text{Weighted empirical Wasserstein distance}}\notag\\
&\quad+\underbrace{C_2}_{\substack{\text{Distribution}\\\text{estimation error}}}
+\underbrace{\sum_{t=1}^T\lambda_t\sum_{i=1}^T\alpha_{t,i}(\xi_{i,t}+\eta_{i,t})}_{\text{Task compatibility}},
\end{align}
where
\begingroup
\allowdisplaybreaks[4]
\begin{align}
\hat{\mathcal L}_{\bm\alpha_t}(h_t)&=\sum_i\alpha_{t,i}\hat{\mathcal L}_i(h_t),\\
\hat{\mathcal L}_i(h_t)&=\frac1{m_i}\sum_{j=1}^{m_i}\ell_i(h_t,x_j^i,y_j^i),\\
C_1&=\left(\sum_i\beta_i\varepsilon_i^2\right)^{1/2},\\
\varepsilon_i&=2\mathfrak R_{m_i}(\mathcal F_i)
+\sqrt{\frac{\log(2T/\delta)}{2m_i}},
\end{align}
\begin{align}
C_2&=2K\sum_t\lambda_t\sum_i\alpha_{t,i}\gamma_{i,t},\\
\intertext{and}
\gamma_{i,t}&=A_i m_i^{-1/s}+A_t m_t^{-1/s}\notag\\*
&\quad+\sqrt{\tfrac12\log(4T^2/\delta)}
\left(m_i^{-1/2}+m_t^{-1/2}\right).
\end{align}
\endgroup
Here, $\{(x_j^i,y_j^i)\}_{j=1}^{m_i}$ are the observations for task $i$, and $\hat{\mathcal D}^i$ is their empirical input distribution. The constants $A_i>0$ and $s>0$ specify the expected Wasserstein convergence rate in Lemma~\ref{lem:wasserstein_concentration}.
The term $C_2$ accounts for finite-sample distribution-estimation error. The joint optimal risk $\xi_{i,t}=\inf_{h\in\mathcal H}[\mathcal L_i(h)+\mathcal L_t(h)]$ characterizes the ability of the hypothesis space to learn tasks $i$ and $t$ jointly.
\end{theorem}
\end{Ronechanged}
\textit{Proof:} See Appendix A.

% \autoref{theorem-Upper bound} 定理引用

\textit{Remark 1:} The weighted empirical loss and the weighted empirical Wasserstein distance terms jointly control the task relation coefficients $\{\bm{\alpha}_t \}_{t=1}^T$.
Specifically, when the empirical loss for task $i$ is small with respect to hypothesis $h_t$, the task relation coefficient $\alpha_{t,i}$ tends to be larger, indicating a stronger relationship between task $t$ and task $i$.
Similarly, when the Wasserstein distance between task $i$ and task $t$ is small, the task relation coefficient $\alpha_{t,i}$ also tends to be larger, indicating a closer similarity between the two tasks.
Therefore, as illustrated in Fig. \ref{fig-TaskSimilarity}, the task similarity of tasks $t$ and $i$ can be quantitatively measured by the Wasserstein distance $W_1\left(\hat{\mathcal{D}}^i, \hat{\mathcal{D}}^t\right)$ and the task relation coefficient $\alpha_{t,i}$.

\begin{figure}[!t]
\centering
\includegraphics[width=0.99\linewidth]{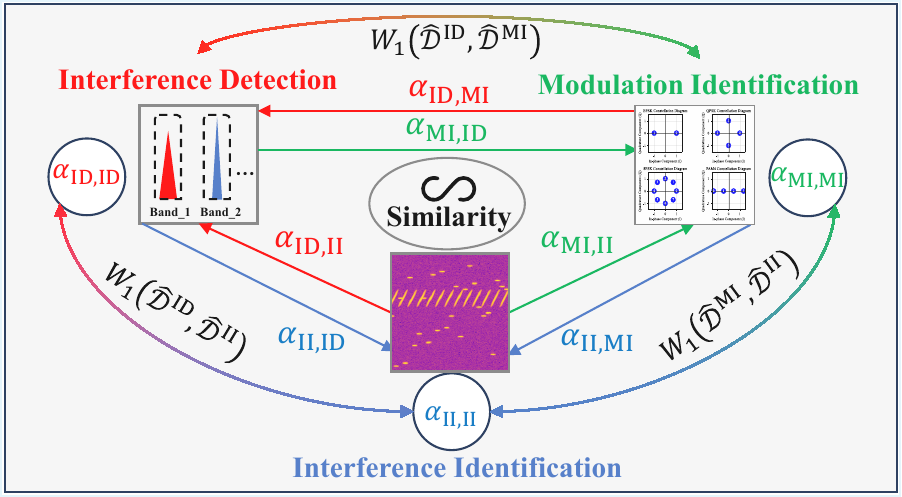}
\caption{Quantitative measurement of task similarity using the Wasserstein distance and task relation coefficients.}
\label{fig-TaskSimilarity}
\end{figure}

\Needspace{6\baselineskip}
\section{Proposed AMTIDIN for Joint Interference Detection and Identification}
\label{sec:proposed_framework}

% R1 CHANGE: 04-objective-introduction
\begin{Ronechanged}
Building on Theorem~\ref{theorem-Upper bound}, we formulate a training objective based on the upper bound on the weighted expected loss.
We then design AMTIDIN and develop an alternating optimization algorithm to jointly learn task hypotheses and task relation coefficients.
\end{Ronechanged}

% The overall architecture and data flow during training of the proposed AMTIDIN are illustrated in Fig. \ref{fig-OverallArchitecture}.
% Guided by the theoretical results established in Section \ref{sec:theoretical_analysis}, 
% we propose the AMTIDIN to jointly achieve interference detection, modulation identification, and interference identification.
% The fundamental objective of AMTIDIN is to simultaneously minimize the weighted empirical loss and the empirical Wasserstein distance across tasks, 
% while adaptively learning the task relation coefficients $\{\bm{\alpha}_t\}_{t=1}^T$. 

\begin{figure*}[!t]
\centering
\includegraphics[width=0.99\linewidth]{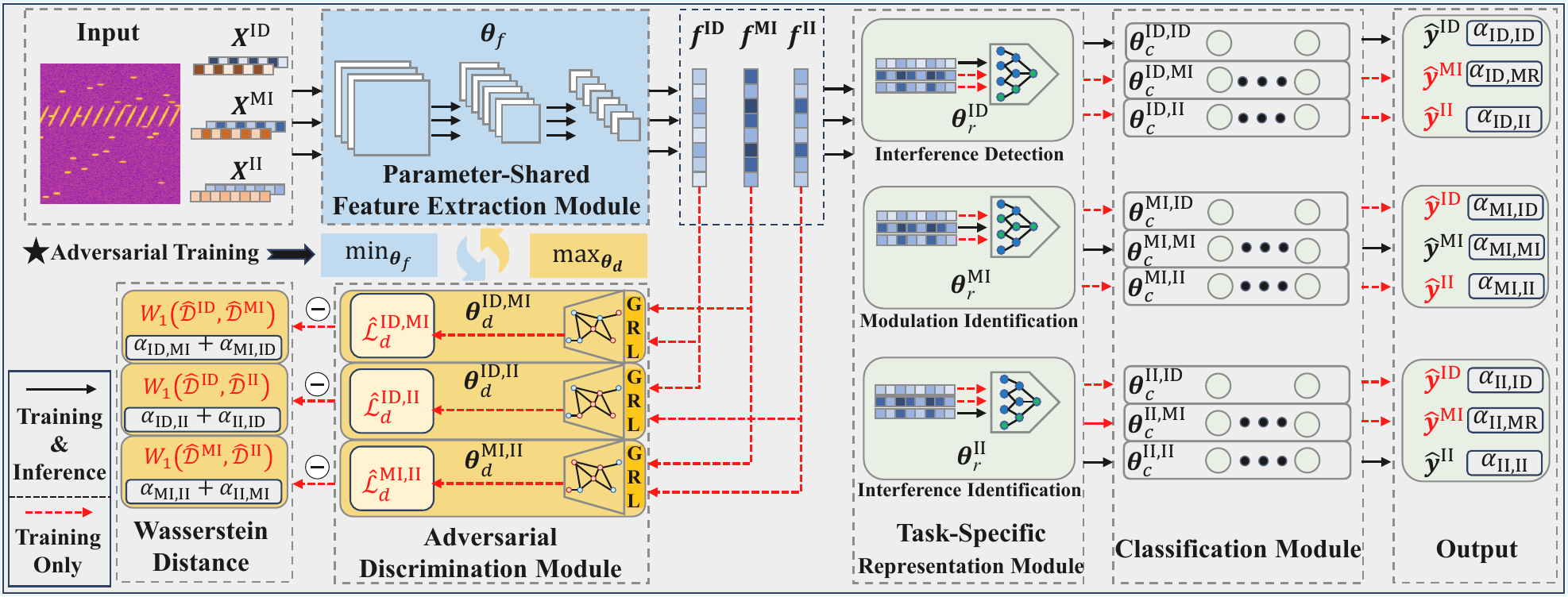}
\caption{
Overall architecture and data flow of the proposed AMTIDIN.
The adversarial discrimination module (yellow block) aims to maximize the empirical Wasserstein distance to distinguish task features, 
while the feature extraction module (blue block) and task-specific representation module (green block) work together to minimize the empirical Wasserstein distance to align task features.
The black arrows denote the data flow operative during both training and inference phases, 
while the red arrows represent the data flow active exclusively during the training phase.}

\label{fig-OverallArchitecture} 
\end{figure*}

\subsection{Problem Transformation}

% R1 CHANGE: cross-entropy bound before the optimization objective
\begin{Ronechanged}
Theorem~\ref{theorem-Upper bound} provides an upper bound on the weighted expected loss in~\eqref{equation-MTL-optimization}.
To obtain a network training objective, we replace the empirical loss term in this bound with cross-entropy.
For each labeled sample,
\begin{equation}
\label{eq:ce_upper_surrogate}
\ell_i(h,x,y)=1-h_y^{(i)}(x)
\leq-\log h_y^{(i)}(x)=\ell_i^{\mathrm{CE}}(h,x,y).
\end{equation}
Let $\hat{\mathcal L}^{\mathrm{CE}}_i(h_t)$ denote the average cross-entropy loss of hypothesis $h_t$ over the $m_i$ labeled samples from task $i$.
Weighting these task losses by $\alpha_{t,i}$ gives $\hat{\mathcal L}^{\mathrm{CE}}_{\bm\alpha_t}(h_t)=\sum_i\alpha_{t,i}\hat{\mathcal L}^{\mathrm{CE}}_i(h_t)$.
Since $\alpha_{t,i}\geq0$, averaging~\eqref{eq:ce_upper_surrogate} over samples and summing over tasks with these weights yields $\hat{\mathcal L}_{\bm\alpha_t}(h_t)\leq\hat{\mathcal L}^{\mathrm{CE}}_{\bm\alpha_t}(h_t)$.
After substituting this upper bound into Theorem~\ref{theorem-Upper bound}, we bound the distribution estimation and task compatibility terms by a constant independent of $h_t$ and $\bm\alpha_t$.\footnote{For a fixed hypothesis space, input distributions, and sample sizes, $\gamma_{i,t}$, $\xi_{i,t}$, and $\eta_{i,t}$ are independent of the selected hypotheses and coefficients. Since $\bm\alpha_t\in\Delta_3$, $C_2+\sum_t\lambda_t\sum_i\alpha_{t,i}(\xi_{i,t}+\eta_{i,t})\leq\sum_t\lambda_t\max_i(2K\gamma_{i,t}+\xi_{i,t}+\eta_{i,t})$. The right-hand side is independent of $h_t$ and $\bm\alpha_t$, so omitting it does not change the minimizers of the resulting upper bound.}
Omitting this additive constant gives the following optimization problem
\end{Ronechanged}
\begin{align}
\label{equation-optimization_problem_original}
\revised{\left(\{h_t^*\}_{t=1}^3,\{\bm\alpha_t^*\}_{t=1}^3\right)}
&= \revised{\mathop{\arg\min}_{h_t\in\mathcal H,\;\bm\alpha_t\in\Delta_3}}
\sum_{t=1}^3 \lambda_t \revised{\hat{\mathcal{L}}^{\mathrm{CE}}}_{\bm{\alpha}_t}\left(h_t\right) \notag \\
&\qquad + \revised{2K} \sum_{t=1}^3 \lambda_t \sum_{i=1}^{3} \alpha_{t,i} W_1\left(\hat{\mathcal{D}}^i, \hat{\mathcal{D}}^t\right) \notag \\
&\qquad + \revised{C_1} \sum_{t=1}^3 \lambda_t \sqrt{\sum_{i=1}^3 \frac{\alpha_{t,i}^2}{\beta_i}}.
\end{align}
\revised{This objective guides the alternating optimization of network parameters and task relation coefficients.}

\subsection{Overall Architecture Design for AMTIDIN}
\revised{Guided by problem~\eqref{equation-optimization_problem_original}, we design AMTIDIN to jointly learn task hypotheses and task relation coefficients.} The task hypotheses $\{h_\text{ID}, h_\text{MI}, h_\text{II}\}$ and adversarial discrimination networks are parameterized by $\bm{\theta} = \{\bm{\theta}_f, \bm{\theta}_r, \bm{\theta}_c, \bm{\theta}_d\}$.
% To solve problem \eqref{equation-optimization_problem_original}, we design the overall architecture for AMTIDIN, where the hypotheses $\{h_\text{ID}, h_\text{MI}, h_\text{II}\}$ are parameterized by the network parameters $\bm{\theta} = \{\bm{\theta}_f, \bm{\theta}_r, \bm{\theta}_c, \bm{\theta}_d\}$.
As illustrated in Fig. \ref{fig-OverallArchitecture}, the architecture of AMTIDIN is structured into four primary modules: the parameter-shared feature extraction module with $\bm{\theta}_f$, the task-specific representation module with $\bm{\theta}_r$, the classification module with $\bm{\theta}_c$, and the adversarial discrimination module with $\bm{\theta}_d$.
\revised{Each hypothesis $h_t$ comprises the shared feature extractor, the task-specific representation module, and the three classification heads. The subscript $r$ denotes representation.} The functionality of each module is summarized as follows:
\begin{enumerate}
\item \textbf{Parameter-Shared Feature Extraction Module} $\big(\bm{\theta}_f\big)$: 
This module is the foundational feature encoder, extracting high-dimensional representations from the received interference signal.
Although the inputs for the three tasks are processed through independent parallel streams, they pass through the shared network layers parameterized by \( \bm{\theta}_f \).
This design enforces the learning of a universal feature space without merging the inputs.
\item \textbf{Task-Specific Representation Module} $\big(\bm{\theta}_r^t, t\in \mathcal{I}_3\big)$: 
This module consists of three parallel sub-networks, each modeling a specific task.
During training, each representation module takes all three feature representations as input and produces a task-specific output representation, which is then fed into the corresponding classification heads.
This design allows each branch to learn from information transferred from other tasks while maintaining its own specialization.
\item \textbf{Classification Module} $\big(\bm{\theta}_c^{t,i}, t, i \in \mathcal{I}_3\big)$:
This module is the final decision-making component of the architecture, comprising a total of nine classification heads, with three heads attached to each representation module. \revised{We write $\bm\theta_c^t=\{\bm\theta_c^{t,i}\}_{i=1}^3$ for the classification parameters of branch $t$.}
During training, we use all nine heads to compute losses and update the task relation coefficients.
Note that, during inference, the final prediction for task $t$ is obtained exclusively from the classification head $\bm{\theta}_c^{t,t}$ associated with its own representation module $\bm{\theta}_r^t$. 
\item \textbf{Adversarial Discrimination Module} $\big(\bm{\theta}_d^{t,i}$, $t, i \in \mathcal{I}_3$, $t<i\big)$: 
This module consists of three discriminators constrained by the 1-Lipschitz condition, each assigned to a unique pair of tasks.
These discriminators compete with the feature extraction module in a min-max game during training. 
Specifically, while the discriminators aim to maximize the estimated Wasserstein distance to distinguish task features, the feature extractor aims to minimize this distance. 
This adversarial competition enforces the learning of robust and invariant representations across tasks.
\end{enumerate}

\subsection{Algorithm Implementation for AMTIDIN}
\revised{To address problem~\eqref{equation-optimization_problem_original}, we train AMTIDIN by alternating between network parameter updates and task relation coefficient updates, as illustrated in Fig.~\ref{fig:algorithm_implementation_flow}. First, fixed coefficients guide network training through classification and adversarial feature alignment. Second, the resulting classification losses and estimated task distances guide coefficient updates with the network fixed.}
\subsubsection{Optimizing Network Parameters with Fixed Task Relation Coefficients}
\label{section-network_parameter_optimization}
\revised{At iteration $\ell$, we fix $\{\bm\alpha_t^{(\ell-1)}\}_{t=1}^3$; the coefficient-regularization term in~\eqref{equation-optimization_problem_original} is then constant with respect to network parameters.}
\revised{In AMTIDIN, we estimate the Wasserstein distance between learned task feature distributions through adversarial training. Let $\hat{\mathcal D}_f^i$ denote the empirical distribution of task-$i$ features produced by the shared extractor with parameters $\bm\theta_f$.}
\revised{Direct computation of Wasserstein distances between high-dimensional feature distributions is difficult. Following the dual formulation in~\eqref{equation-Wasserstein distance}, we introduce a discriminator $d^{t,i}$ for each task pair $t<i$ and define}
\begin{align}
\label{eq:discriminator_loss}
\hat{\mathcal{L}}_d^{t,i}\left(\bm{\theta}_f,\bm{\theta}_d^{t,i}\right) 
&= \mathbb{E}_{\bm{X}^t \sim \hat{\mathcal{D}}^t} d^{t,i}\left(\bm{X}^t; \bm{\theta}_f, \bm{\theta}_d^{t,i}\right) \notag \\
&\quad - \mathbb{E}_{\bm{X}^i \sim \hat{\mathcal{D}}^i} d^{t,i}\left(\bm{X}^i; \bm{\theta}_f, \bm{\theta}_d^{t,i}\right),
\end{align}
where $d^{t,i}\left(\bm{X}; \bm{\theta}_f, \bm{\theta}_d^{t,i}\right)$ is the output of the adversarial discrimination network for task pair $(t,i)$. \revised{In the following sums, we set $\hat{\mathcal L}_d^{t,t}=0$ and reuse each pairwise estimate as $\hat{\mathcal L}_d^{i,t}=\hat{\mathcal L}_d^{t,i}$ for $t<i$, with $\bm\theta_d^{i,t}=\bm\theta_d^{t,i}$.}
To satisfy the 1-Lipschitz constraint required for Wasserstein distance estimation, we apply spectral normalization (SN) \cite{Miyato2018Spectral} to the weights of each adversarial discrimination network $d^{t,i}$. 

\begin{figure}[!t]
\centering
\includegraphics[width=0.99\linewidth]{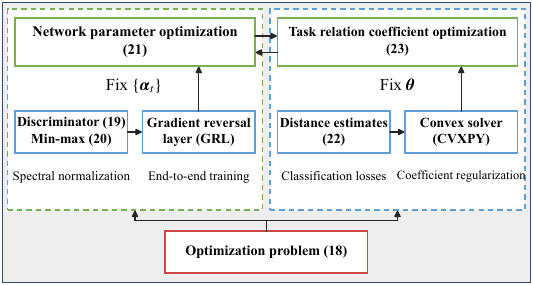}
\caption{
\revised{Alternating optimization of network parameters and task relation coefficients in AMTIDIN.}}
\label{fig:algorithm_implementation_flow}
\end{figure}

\begin{figure}[!b]
\centering
\includegraphics[width=0.99\linewidth]{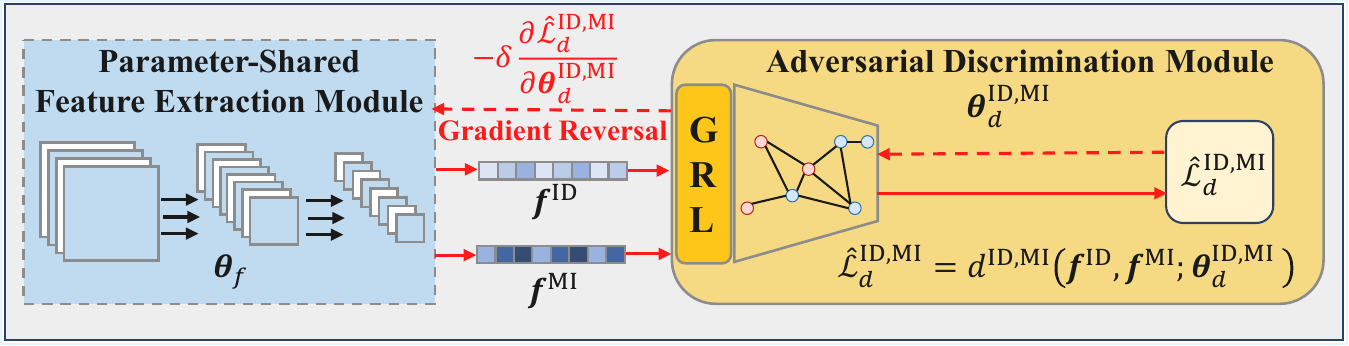}
\caption{Gradient Reversal Layer (GRL) for adversarial learning. 
The red dashed lines represent the gradient backpropagation flow for the adversarial loss.
}
\label{fig:GradientReversal}
\end{figure}
\revised{With the task relation coefficients fixed, we combine the cross-entropy term in~\eqref{equation-optimization_problem_original} with the discriminator objective in~\eqref{eq:discriminator_loss} to optimize the network parameters. The weight $\rho\geq0$ balances classification and adversarial feature alignment}
\begin{align}
\label{equation-optimization_problem_min_max}
&\revised{\{\bm{\theta}_f^{(\ell)}, \bm{\theta}_r^{(\ell)}, \bm{\theta}_c^{(\ell)}\}} \notag \\
&\quad= \arg\min_{\bm{\theta}_f, \bm{\theta}_r, \bm{\theta}_c} \max_{\bm{\theta}_{d}} \sum_{t=1}^3 \lambda_t \revised{\hat{\mathcal{L}}^{\mathrm{CE}}}_{\bm{\alpha}_t^{(\ell-1)}}\left(\bm{\theta}_f,\bm{\theta}_r^{t},\bm{\theta}_c^{t}\right) \notag \\
&\quad\quad+ \rho \sum_{t=1}^3 \lambda_t \sum_{i=1}^3 \alpha_{t,i}^{(\ell-1)} \hat{\mathcal{L}}_{d}^{t,i}\left(\bm{\theta}_f,\bm{\theta}_{d}^{t,i}\right).
\end{align}
\revised{Solving this min--max problem typically involves alternating updates of the discriminator and feature extractor, which can lead to training instability.}
\revised{To implement~\eqref{equation-optimization_problem_min_max} through end-to-end backpropagation, we use a gradient reversal layer (GRL)~\cite{Ganin2016Domainadversarial}, as illustrated in Fig.~\ref{fig:GradientReversal}. The GRL leaves features unchanged during the forward pass and reverses the gradient toward the feature extractor during backpropagation. With this gradient rule, the network training expression becomes}
\begin{align}
\label{equation-optimization_problem_minimization}
&\{\bm{\theta}_f^{(\ell)}, \bm{\theta}_r^{(\ell)}, \bm{\theta}_c^{(\ell)}, \bm{\theta}_{d}^{(\ell)}\} \notag \\
&\quad= \arg\min_{\bm{\theta}_f, \bm{\theta}_r, \bm{\theta}_c, \bm{\theta}_{d}} \sum_{t=1}^3 \lambda_t \revised{\hat{\mathcal{L}}^{\mathrm{CE}}}_{\bm{\alpha}_t^{(\ell-1)}}\left(\bm{\theta}_f,\bm{\theta}_r^{t},\bm{\theta}_c^{t}\right) \notag \\
&\quad\quad+ \rho \sum_{t=1}^3 \lambda_t \sum_{i=1}^3 \alpha_{t,i}^{(\ell-1)} \hat{\mathcal{L}}_{d_\text{GRL}}^{t,i}\left(\bm{\theta}_f,\bm{\theta}_{d}^{t,i}\right),
\end{align}
\revised{Here, the forward loss is $\hat{\mathcal L}_{d_\text{GRL}}^{t,i}=-\hat{\mathcal L}_d^{t,i}$. The GRL reverses its gradient only toward the feature extractor, so gradient descent maximizes $\hat{\mathcal L}_d^{t,i}$ for the discriminator and minimizes it for the feature extractor.} 
We can solve problem \eqref{equation-optimization_problem_minimization} using mini-batch stochastic gradient descent with the Adam optimizer.
\revised{After updating the network parameters, we use the trained discriminators to estimate the Wasserstein distances between the task feature distributions induced by $\bm\theta_f^{(\ell)}$}
\begin{align}
\label{equation-Wasserstein distance_estimation}
\revised{W_1\left(\hat{\mathcal D}_f^i,\hat{\mathcal D}_f^t\right)}
&\revised{\approx\hat{\mathcal L}_d^{t,i}\left(\bm\theta_f^{(\ell)},\bm\theta_d^{t,i(\ell)}\right)}\notag\\
&\revised{=-\hat{\mathcal L}_{d_\text{GRL}}^{t,i}\left(\bm\theta_f^{(\ell)},\bm\theta_d^{t,i(\ell)}\right).}
\end{align}

\begin{algorithm}[!b]
\caption{Training Algorithm of AMTIDIN}
\label{algorithm1}
\renewcommand{\algorithmicrequire}{\textbf{Input:}}
\renewcommand{\algorithmicensure}{\textbf{Output:}}
\begin{algorithmic}[1]
\REQUIRE ~\\ 
\hspace*{-2em} \revised{Training set $\mathcal S_{\mathrm{tr}}$; Validation set $\mathcal S_{\mathrm{val}}$;} \\
\hspace*{-2em} Task weights $\{\lambda_t\}_{t=1}^3$;\\
\hspace*{-2em} \revised{Distance weights $\rho$, $\kappa$; Regularization weight $\tau$;} \\
\hspace*{-2em} \revised{Early-stopping patience $P_{\mathrm{stop}}$; Max epochs $E$; Batch size $B$.} \\
\ENSURE ~\\ 
\hspace*{-2em} \revised{Selected network parameters $\bm\theta^{(\ell^\star)}$;} \\
\hspace*{-2em} \revised{Task relation coefficients $\{\bm\alpha_t^{(\ell^\star)}\}_{t=1}^3$.}
\STATE Initialize $\bm{\theta}$ randomly and $\bm{\alpha}_t$ as one-hot vectors;
\STATE \revised{Initialize the best weighted validation loss to $+\infty$;}

\FOR{\revised{epoch $\ell=1,\ldots,E$}}
   \FOR{each mini-batch $\mathcal{B} \subset \revised{\mathcal{S}_{\mathrm{tr}}}$ of size $B$}
       \STATE Update $\bm{\theta}$ via solving \eqref{equation-optimization_problem_minimization} using Adam optimizer;
   \ENDFOR % <--- 修改点：Batch 循环在这里结束
   
   % 系数更新移动到 Batch 循环之外，Epoch 循环之内
   \FOR{each task $t \in \mathcal{I}_3$}
       \STATE Update $\bm{\alpha}_t$ via solving \eqref{eq:coefficient_optimization} using \texttt{CVXPY}.
   \ENDFOR
   \STATE \revised{Evaluate the weighted validation loss on $\mathcal S_{\mathrm{val}}$;}
   \IF{\revised{the weighted validation loss improves}}
       \STATE \revised{Update the best validation loss;}
\STATE \revised{Set $\ell^\star=\ell$ and save $\bm\theta$ and $\{\bm\alpha_t\}_{t=1}^3$;}
   \ENDIF
   \IF{\revised{$\ell-\ell^\star\geq P_{\mathrm{stop}}$}}
       \STATE \revised{\textbf{break};}
   \ENDIF
\ENDFOR
\end{algorithmic}
\end{algorithm}

\subsubsection{Optimizing Task Relation Coefficients with Fixed Network Parameters}
\revised{With the updated network parameters $\bm\theta^{(\ell)}$ fixed, we use the cross-entropy term in~\eqref{equation-optimization_problem_original} and the Wasserstein distance estimates in~\eqref{equation-Wasserstein distance_estimation} to update the task relation coefficients.
We replace $2K$ and $C_1$ with training hyperparameters $\kappa\geq0$ and $\tau\geq0$, respectively, to control distance weighting and coefficient regularization.
We obtain $\{\bm\alpha_t\}_{t=1}^3$ by solving the following convex optimization problem subject to $\bm\alpha_t\in\Delta_3$ for each task $t$}
\begin{align}
\label{eq:coefficient_optimization}
\{\bm{\alpha}_t^{(\ell)}\}_{t=1}^3 &= \arg\min_{\{\bm{\alpha}_t\}_{t=1}^3} \sum_{t=1}^3 \lambda_t \revised{\hat{\mathcal{L}}^{\mathrm{CE}}}_{\bm{\alpha}_t}\left(\bm{\theta}_f^{(\ell)},\bm{\theta}_r^{t(\ell)},\bm{\theta}_c^{t(\ell)}\right) \notag \\
&\quad + \revised{\kappa} \sum_{t=1}^3 \lambda_t \sum_{i=1}^3 \alpha_{t,i} \revised{W_1\left(\hat{\mathcal D}_f^i,\hat{\mathcal D}_f^t\right)} \notag \\
&\quad + \revised{\tau} \sum_{t=1}^3 \lambda_t \sqrt{\sum_{i=1}^3 \frac{\alpha_{t,i}^2}{\beta_i}}.
\end{align}
\revised{With the network fixed, the classification and distance terms in~\eqref{eq:coefficient_optimization} are linear in $\bm\alpha_t$, and the regularization term is convex. We solve this convex subproblem over the simplex using \texttt{CVXPY}~\cite{Diamond2016CVXPY}.}
\revised{Training runs for at most $E$ epochs and stops early if the weighted validation loss does not decrease for $P_{\mathrm{stop}}$ consecutive epochs, where $P_{\mathrm{stop}}$ is the early-stopping patience. Let $\ell^\star$ denote the epoch with the lowest weighted validation loss among the completed epochs. We retain the network parameters $\bm\theta^{(\ell^\star)}$ and task relation coefficients $\{\bm\alpha_t^{(\ell^\star)}\}_{t=1}^3$ from that epoch.}
\revised{The final distance estimates from~\eqref{equation-Wasserstein distance_estimation} and learned coefficients are used for the task similarity analysis in Section~\ref{sec:experimental_results}. Algorithm~\ref{algorithm1} summarizes both stages, updating the network over mini-batches before updating the coefficients at each epoch.}

\begin{figure*}[!t]
\centering
\includegraphics[width=0.92\linewidth]{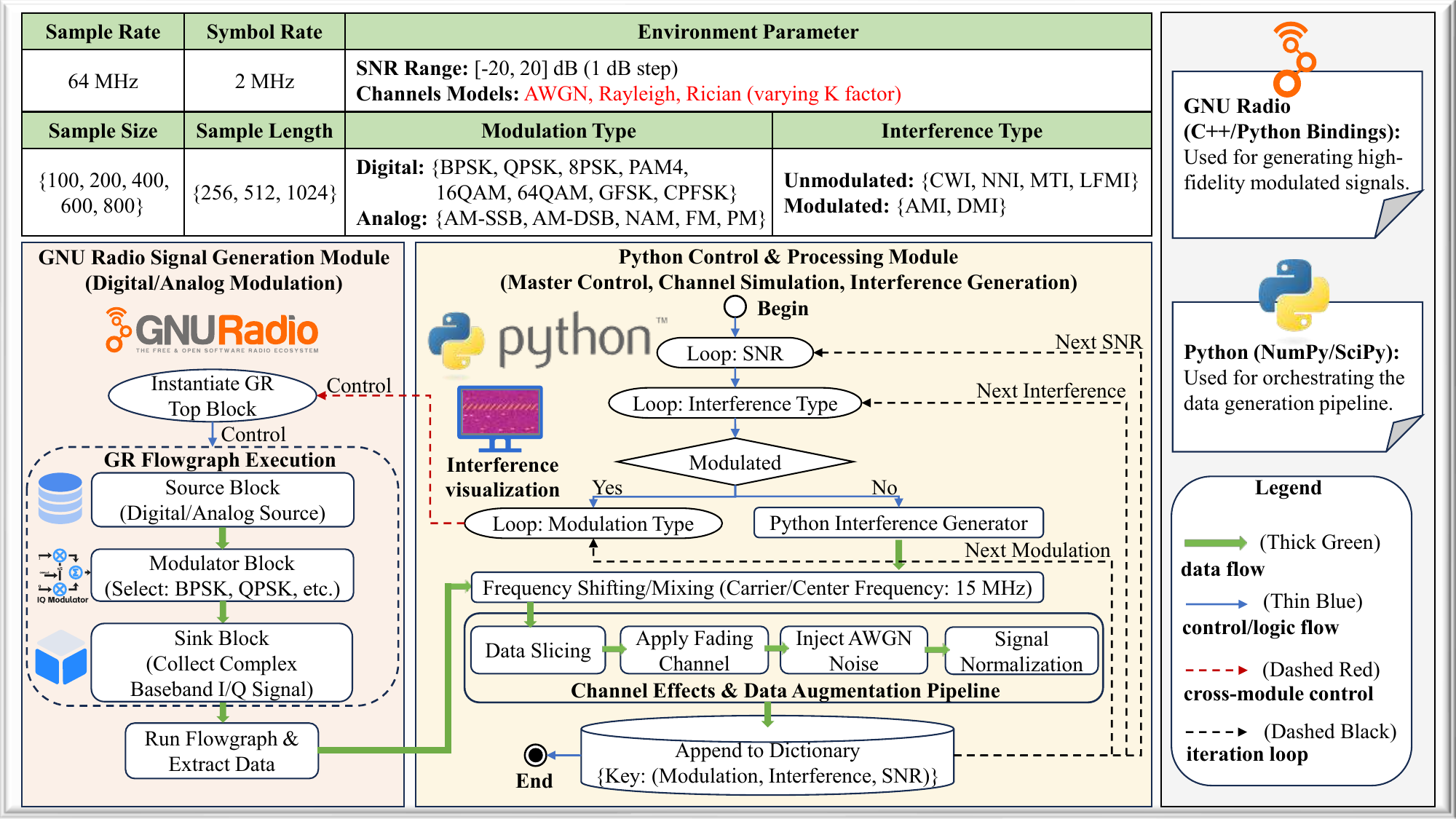}
\caption{Overview of the dataset generation framework. 
The framework illustrates the hybrid pipeline where GNU Radio handles modulated waveform generation while Python manages interference injection and channel modeling.}
\label{fig-overall_data_generation} 
\end{figure*}

\section{Experimental Results}
\label{sec:experimental_results}
\revised{We evaluate AMTIDIN in terms of classification performance, task relationship characterization, and training convergence. Comparisons with STL and MTL baselines use the same architectures for corresponding network modules to assess the benefits of the proposed learning framework.}
\subsection{Experimental Setup}
\label{subsec:experimental_setup}
\subsubsection{Dataset Generation}
As illustrated in Fig. \ref{fig-overall_data_generation}, we constructed a multi-task dataset using a hybrid generation framework. 
The top panel of Fig. \ref{fig-overall_data_generation} details the specific parameter configurations.
In this system, Python scripts function as both the master controller and the generator for unmodulated interference, while the GNU Radio software-defined radio (SDR) platform is used for modulated waveform generation. 
For modulation identification, we adopted the methodology from the RadioML 2016.10a dataset \cite{Oshea2016Radio}, extending the original set to include noise amplitude modulation (NAM) and phase modulation (PM).
Regarding interference identification, because of the limited availability of open-source datasets, we generated six common interference types based on \cite{Wang2024Wireless}: continuous wave interference (CWI), narrowband noise interference (NNI), multi-tone interference (MTI), linear frequency modulation interference (LFMI), digital modulation interference (DMI), and analog modulation interference (AMI).
\revised{To evaluate performance under noise and fading, AWGN, Rayleigh, and Rician channel models were employed in equal proportions at each SNR level.}
For the interference detection task, a balanced dataset was constructed by using the previously generated interference mixture samples as the positive class, supplemented by an equal number of pure Gaussian noise samples as the negative class. 
Notably, the dataset supports diverse configurations of sample size, signal length, and SNR range to accommodate various experimental requirements. 
Furthermore, each sample is annotated with multiple labels (modulation, interference, SNR, and presence) to facilitate joint interference detection and identification tasks.

\begin{table}[!t]
\centering
\scriptsize
\caption{Detailed Architecture of the Proposed AMTIDIN}
\label{tab:network_architecture}
\renewcommand{\arraystretch}{0.95}
\setlength{\tabcolsep}{2pt}
\renewcommand{\tabularxcolumn}[1]{m{#1}}
\begin{tabularx}{\linewidth}{
    >{\hsize=0.75\hsize\centering\arraybackslash}X 
    >{\hsize=0.7\hsize\centering\arraybackslash}X 
    >{\hsize=1.6\hsize\centering\arraybackslash}X 
    >{\hsize=0.95\hsize\centering\arraybackslash}X
}
\toprule
\textbf{Module} & \textbf{Layer} & \textbf{Parameters} & \textbf{Output} \\
\midrule
% 特征提取
\multirow{11}{*}{\makecell[c]{Parameter-\\Shared \\ Feature \\ Extraction \\ $\big(\bm{\theta}_f\big)$}} 
& Input     & $\{\bm{X^\text{ID}}, \bm{X^\text{MI}}, \bm{X^\text{II}}\}$ & $\left(3, 2, N\right)$ \\
& Conv1D-1  & $2 \!\to\! 32, k\!=\!3, s\!=\!1, p\!=\!1$ & $\left(3, 32, N\right)$ \\
& BatchNorm & $\text{features}\!=\!32$ & $\left(3, 32, N\right)$ \\
& GELU      & -- & $\revised{\left(3, 32, N\right)}$ \\
& Conv1D-2  & $32\! \to\! 64, k\!=\!5, s\!=\!1, p\!=\!2$ & $\left(3, 64, N\right)$ \\
& BatchNorm & $\text{features}\!=\!64$ & $\left(3, 64, N\right)$ \\
& GELU      & -- & $\left(3, 64, N\right)$ \\
& Conv1D-3  & $64 \!\to\! 128, k\!=\!7, s\!=\!1, p\!=\!3$ & $\left(3, 128, N\right)$ \\
& BatchNorm & $\text{features}\!=\!128$ & $\left(3, 128, N\right)$ \\
& GELU      & -- & $\left(3, 128, N\right)$ \\
& AvgPool   & $\text{out\_size}\!=\!1$ & $\left(3, 128\right)$ \\
\midrule
% 假设空间
\multirow{6}{*}{\makecell[c]{Representation \\ $\big(\bm{\theta}_r^t\big)$}}
& Input     &  $\{\bm{f}^\text{ID}, \bm{f}^\text{MI}, \bm{f}^\text{II}\}$ & $\left(3, 128\right)$ \\
& Residual  & $3\! \times \!\left(128 \!\to \!256\right)$, $\text{drop}\!=\!0.1$ & $\left(3, 256\right)$ \\
& Linear-1  & $3 \!\times\! \left(256 \!\to\! 128\right)$ & $\left(3, 128\right)$ \\
& GELU      & -- & $\left(3, 128\right)$ \\
& Linear-2  & $3 \!\times \!\left(128 \!\to\! 64\right)$ & $\left(3, 64\right)$ \\
& GELU      & -- & $\left(3, 64\right)$ \\
\midrule
% 分类层
\multirow{3}{*}{\makecell[c]{Classification \\ $\big(\bm{\theta}_c^{t,i}\big)$}} 
& Input     & Parallel representation out & $\left(3, 64\right)$ \\
& Linear    & $\revised{64\to c_i}$ & $\revised{(c_i)}$ \\
& Softmax   & -- & $\revised{(c_i)}$ \\
\midrule
% 对抗判别
\multirow{6}{*}{\makecell[c]{Adversarial \\ Discrimination \\ $\big(\bm{\theta}_d^{t,i}$, $t<i\big)$}}
& Input     & \{$\bm{f}^t, \bm{f}^i$\} & $\left(2, 128\right)$ \\   
& GRL       & Gradient Reversal & $\left(2, 128\right)$ \\
& SN Linear-1 & $2\! \times \!\left(128 \!\to\! 64\right)$ & $\left(2, 64\right)$ \\
& ELU       & -- & $\left(2, 64\right)$ \\
& SN Linear-2 & $2\! \times\! \left(64 \!\to\! 1\right)$ & $\left(2, 1\right)$ \\
& Sigmoid   & -- & $\left(2, 1\right)$ \\
\bottomrule
\end{tabularx}
\par\smallskip
\begin{minipage}{\linewidth}\scriptsize
\revised{The leading dimension counts task streams; batch dimensions are omitted. Each of the nine classification heads outputs $c_i$ class probabilities, with $i=1,2,3$ corresponding to ID, MI, and II. In convolutional layers, $k$, $s$, and $p$ denote kernel size, stride, and padding; drop denotes dropout probability. GELU and ELU denote Gaussian error linear unit and exponential linear unit, respectively.}
\end{minipage}
\end{table}
\vspace{-1em}

\subsubsection{Implementation Details}
\label{subsec:implementation_details}
The proposed AMTIDIN was implemented using Python 3.11 and PyTorch 2.5.1. All experiments were conducted on a high-performance workstation equipped with an Intel Core i9-13900HX CPU, 16 GB RAM, and an NVIDIA GeForce RTX 4060 GPU (8 GB VRAM). 
The detailed network architecture is summarized in Table \ref{tab:network_architecture}. 
During the training phase, the Adam optimizer was employed with a dynamic learning rate adjustment strategy via the \texttt{ReduceLROnPlateau} scheduler. 
Specifically, the initial learning rate was set to 0.001, with a decay factor of 0.1 and a patience of 8 epochs, down to a minimum of $1.0 \times 10^{-7}$. 
The batch size and maximum training epochs were configured as $B = 256$ and $E = 100$, respectively.
\revised{Learning rate scheduling and early stopping monitor the sum of the three tasks' validation cross-entropy losses weighted by $\lambda_t$. The early-stopping patience is set to $P_{\mathrm{stop}}=16$ epochs, and the checkpoint with the lowest weighted validation loss at $\ell^\star$ is retained.}
Regarding the optimization objective, task weights were assigned as $\lambda_{\text{ID}} = 0.05$, $\lambda_{\text{MI}} = \revised{0.80}$, and $\lambda_{\text{II}} = 0.15$ to prioritize modulation identification. \revised{We choose an adversarial alignment weight of $\rho=0.1$ to balance distribution alignment with the preservation of task-discriminative features. Excessive alignment pressure may suppress features needed by individual tasks and compromise classification performance. The regularization weight $\tau$ and the distance weight $\kappa$ for coefficient updates are set to 1 and 0.7, respectively.}
Finally, the dataset was partitioned into training, validation, and test sets with a split of 60\%, 20\%, and 20\%.

\subsubsection{Evaluation Metric}
To quantitatively evaluate the proposed method, we adopt classification accuracy as the primary metric. For a specific task $t$, $\text{Accuracy}^t$ measures the percentage of correctly classified samples relative to the total test set size $M_{\text{test}}$, formulated as
\begin{equation}
    \text{Accuracy}^t = \frac{1}{M_{\text{test}}} \sum_{j=1}^{M_{\text{test}}} \mathbb{I}\left(\hat{y}_j^t = y_j^t\right) \times 100\%,
    \label{eq:accuracy}
\end{equation}
where $\hat{y}_j^t$ and $y_j^t$ denote the predicted and ground-truth labels for the $j$-th sample, respectively. 
All reported results are averaged over five independent runs.
% \subsubsection{Baseline Models}

\subsection{Task Similarity Analysis}
To validate the proposed similarity metrics, we employed a restricted training regime by limiting the sample size to 100 per class and the signal length to 256. 
This constraint intentionally compels the model to exploit intrinsic task correlations for performance rather than relying solely on data abundance. 
We trained AMTIDIN and evaluated the empirical Wasserstein distance and task relation coefficients under low (-5 dB), medium (5 dB), and high (15 dB) SNR.
The Wasserstein distance matrices before and after sigmoid transformation, along with the task relation coefficient matrices, are illustrated in Figs.~\ref{wasserstein_distance_matrix_final_logits}, \ref{wasserstein_distance_matrix_final}, and \ref{task_relation_coefficient_matrix_final}, respectively.

\begin{figure}[!b]
\vspace{-1em}
\centering
\subfloat[SNR = -5 dB]{\includegraphics[width=0.31\linewidth]{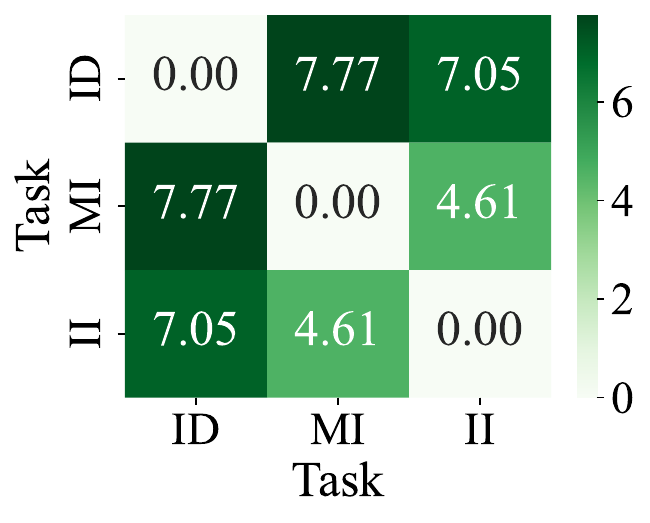}%
\label{fig-wasserstein_distance_matrix_final_logits_-5}}
\hfil
\subfloat[SNR = 5 dB]{\includegraphics[width=0.33\linewidth]{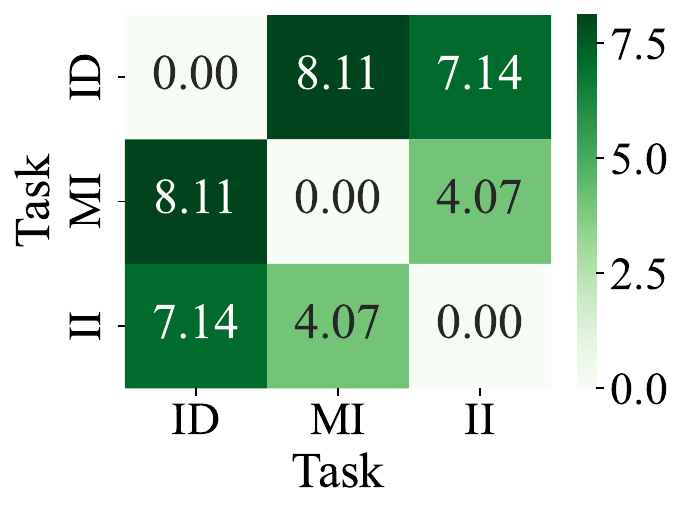}%
\label{fig-wasserstein_distance_matrix_final_logits_5}}
\hfil
\subfloat[SNR = 15 dB]{\includegraphics[width=0.31\linewidth]{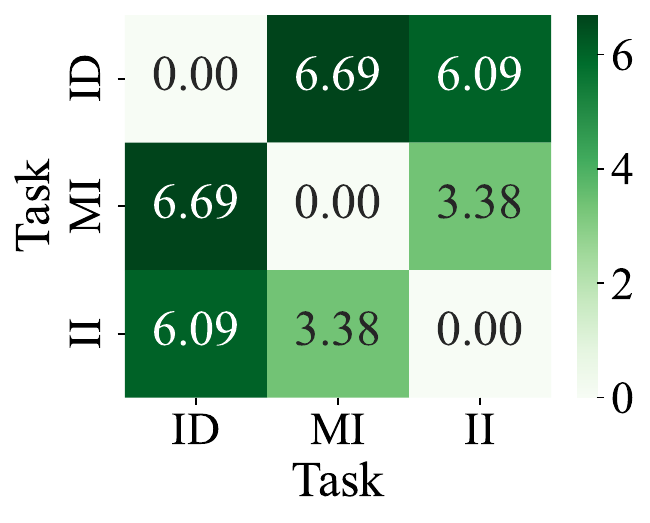}%
\label{fig-wasserstein_distance_matrix_final_logits_15}}
\caption{Wasserstein distance matrices before sigmoid transformation. (a) SNR = -5 dB. (b) SNR = 5 dB. (c) SNR = 15 dB.}
\label{wasserstein_distance_matrix_final_logits}
\vspace{-1em}
\end{figure}

\begin{figure}[!b]
\centering
\subfloat[SNR = -5 dB]{\includegraphics[width=0.33\linewidth]{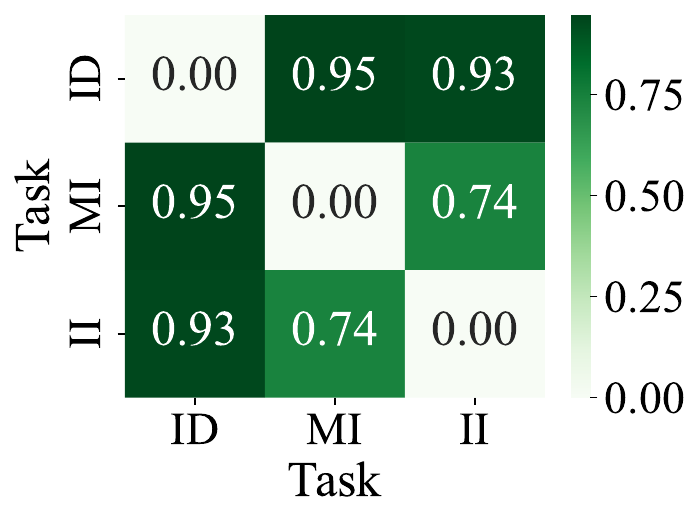}%
\label{fig-wasserstein_distance_matrix_final_-5}}
\hfil
\subfloat[SNR = 5 dB]{\includegraphics[width=0.33\linewidth]{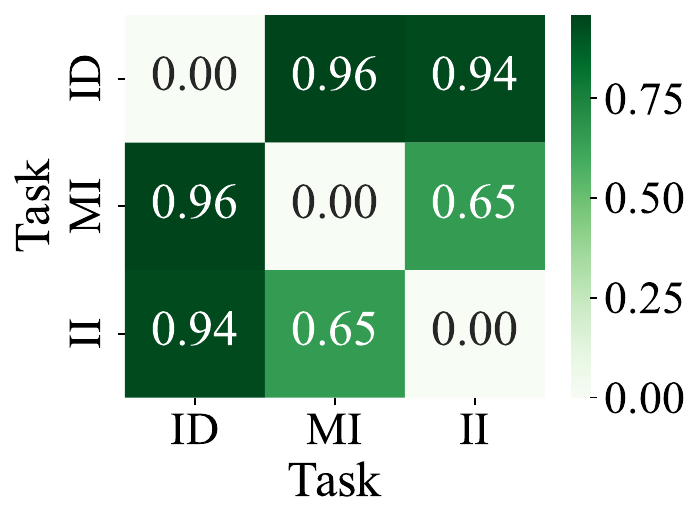}%
\label{fig-wasserstein_distance_matrix_final_5}}
\hfil
\subfloat[SNR = 15 dB]{\includegraphics[width=0.33\linewidth]{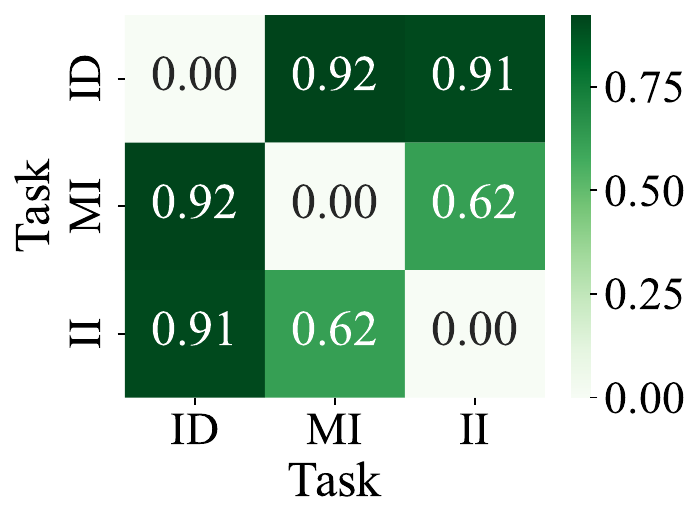}%
\label{fig-wasserstein_distance_matrix_final_15}}
\caption{Wasserstein distance matrices after sigmoid transformation. (a) SNR = -5 dB. (b) SNR = 5 dB. (c) SNR = 15 dB.}
\label{wasserstein_distance_matrix_final}
\vspace{-1em}
\end{figure}

\begin{figure}[!b]
\centering
\subfloat[SNR = -5 dB]{\includegraphics[width=0.33\linewidth]{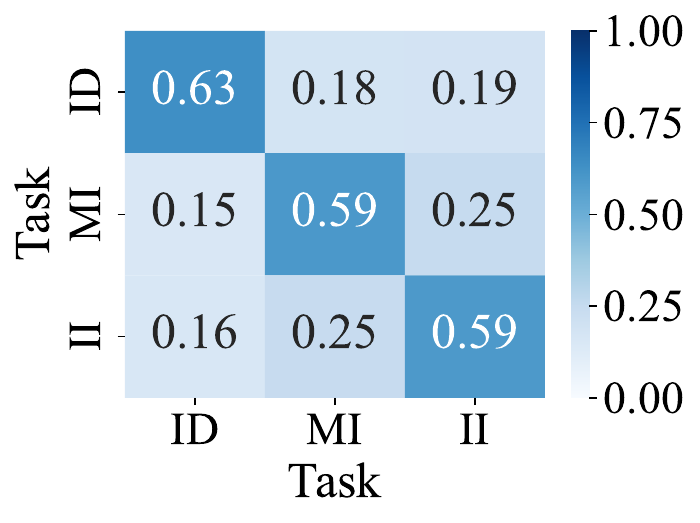}%
\label{fig-task_relation_coefficient_matrix_final_-5}}
\hfil
\subfloat[SNR = 5 dB]{\includegraphics[width=0.33\linewidth]{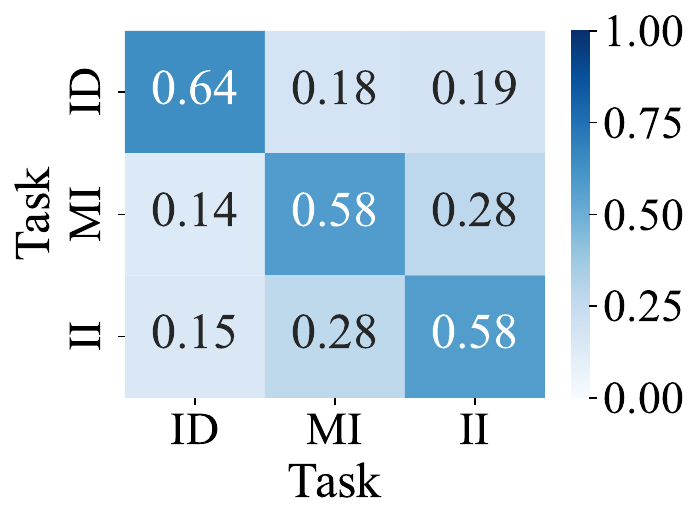}%
\label{fig-task_relation_coefficient_matrix_final_5}}
\hfil
\subfloat[SNR = 15 dB]{\includegraphics[width=0.33\linewidth]{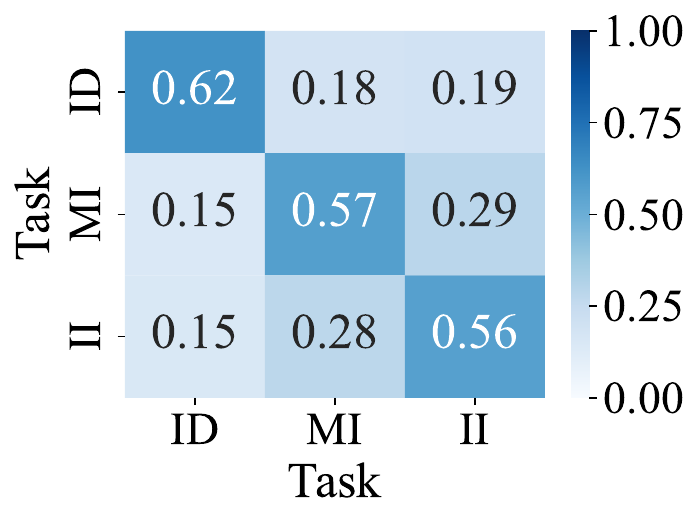}%
\label{fig-task_relation_coefficient_matrix_final_15}}
\caption{Task relation coefficient matrices. (a) SNR = -5 dB. (b) SNR = 5 dB. (c) SNR = 15 dB.}
\label{task_relation_coefficient_matrix_final}
\vspace{-1em}
\end{figure}

As summarized in Table \ref{tab:task_similarity}, the results reveal a significant similarity between modulation identification and interference identification. 
These two tasks exhibit the lowest average Wasserstein distance (4.02) and the highest task relation coefficients (0.27/0.27). 
This indicates that both tasks share a substantial overlap in the fine-grained feature space, facilitating mutual enhancement during joint learning.
By contrast, interference detection demonstrates a marked distributional divergence from the other tasks, evidenced by significantly larger distances (averaging 7.52 with modulation identification and 6.76 with interference identification) and lower relation coefficients (averaging 0.18/0.15 with modulation identification and 0.19/0.15 with interference identification). 
This disparity aligns with the physical intuition that interference detection is fundamentally a coarse-grained binary task focusing on signal presence, whereas modulation identification and interference identification require detailed pattern discrimination. 
These observations confirm that our method effectively quantifies the underlying relationships among interference detection and identification tasks, clarifying why feature transfer is more robust between modulation identification and interference identification than with interference detection.

\begin{table}[!t]
\centering
\caption{Quantitative Results of Task Similarity}
\label{tab:task_similarity}
\renewcommand{\arraystretch}{0.85}
\setlength{\tabcolsep}{2pt}
\setlength{\aboverulesep}{1.5pt}
\setlength{\belowrulesep}{1.5pt}
\begin{tabularx}{\linewidth}{
    >{\hsize=0.8\hsize\centering\arraybackslash}X 
    >{\hsize=0.6\hsize\centering\arraybackslash}X 
    >{\hsize=1.0\hsize\centering\arraybackslash}X 
    >{\hsize=1.1\hsize\centering\arraybackslash}X
    >{\hsize=1.5\hsize\centering\arraybackslash}X
}
\toprule
\textbf{Task Pair} & \textbf{SNR} & \textbf{W-Dist} & \textbf{W-Dist (Sig)} & \textbf{Relation Coeff} \\
\midrule

% 第一组: ID & MI
\multirow{4}{*}{\textbf{ID \& MI}} 
 & -5 dB & 7.77 & 0.95 & 0.18 / 0.15 \\
 & 5 dB  & 8.11 & 0.96 & 0.18 / 0.14 \\
 & 15 dB & 6.69 & 0.92 & 0.18 / 0.15 \\
 & \textbf{Avg} & \textbf{7.52} & \textbf{0.94} & \textbf{0.18 / 0.15} \\
\midrule

% 第二组: ID & II
\multirow{4}{*}{\textbf{ID \& II}} 
 & -5 dB & 7.05 & 0.93 & 0.19 / 0.16 \\
 & 5 dB  & 7.14 & 0.94 & 0.19 / 0.15 \\
 & 15 dB & 6.09 & 0.91 & 0.19 / 0.15 \\
 & \textbf{Avg} & \textbf{6.76} & \textbf{0.93} & \textbf{0.19 / 0.15} \\
\midrule

% 第三组: MI & II
\multirow{4}{*}{\textbf{MI \& II}} 
 & -5 dB & 4.61 & 0.74 & 0.25 / 0.25 \\
 & 5 dB  & 4.07 & 0.65 & 0.28 / 0.28 \\
 & 15 dB & 3.38 & 0.62 & 0.29 / 0.28 \\
 & \textbf{Avg} & \textbf{4.02} & \textbf{0.67} & \textbf{0.27 / 0.27} \\
\bottomrule
\end{tabularx}
\par\smallskip
\begin{minipage}{\linewidth}\footnotesize
\revised{For each listed task pair $(t,i)$, relation coefficients are reported as $\alpha_{t,i}/\alpha_{i,t}$.}
\end{minipage}
\end{table}

\subsection{Comparison with Task-Specific STL Baseline}
\label{subsec:comparison_stl}
To rigorously validate that the performance gains stem from our proposed adversarial multi-task framework rather than architectural complexity, we compare AMTIDIN against its task-specific STL baseline (AMTIDIN-STL).
\revised{AMTIDIN-STL trains a separate single-task network for each task, using the same feature-extraction, representation, and classification architectures as the corresponding AMTIDIN path. Each network is trained independently end to end.}
\revised{We evaluate the benefits of joint learning by comparing joint and independent training with the same architecture for each task and consistent training hyperparameters.}

\begin{figure*}[!b]
\centering
\subfloat[Accuracy of interference detection]{\includegraphics[width=0.32\linewidth]{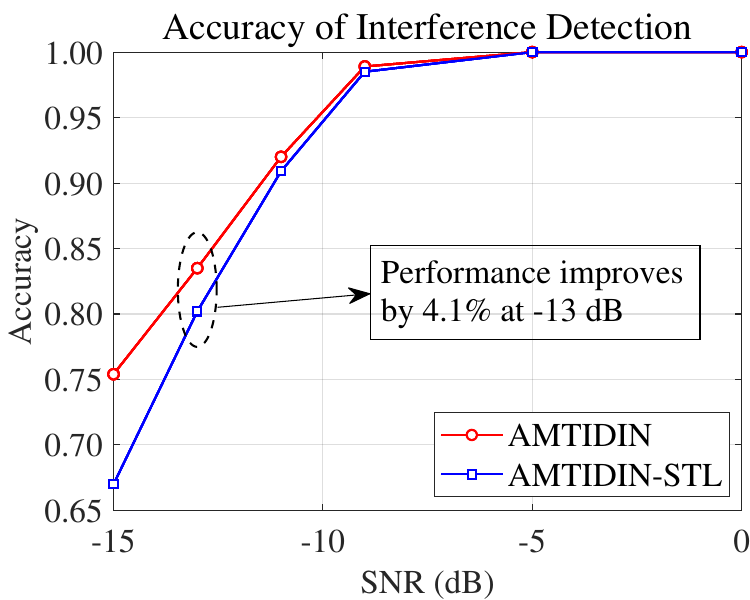}%
\label{fig-M_S_SNR_ID}}
\hfil
\subfloat[Accuracy of modulation identification]{\includegraphics[width=0.32\linewidth]{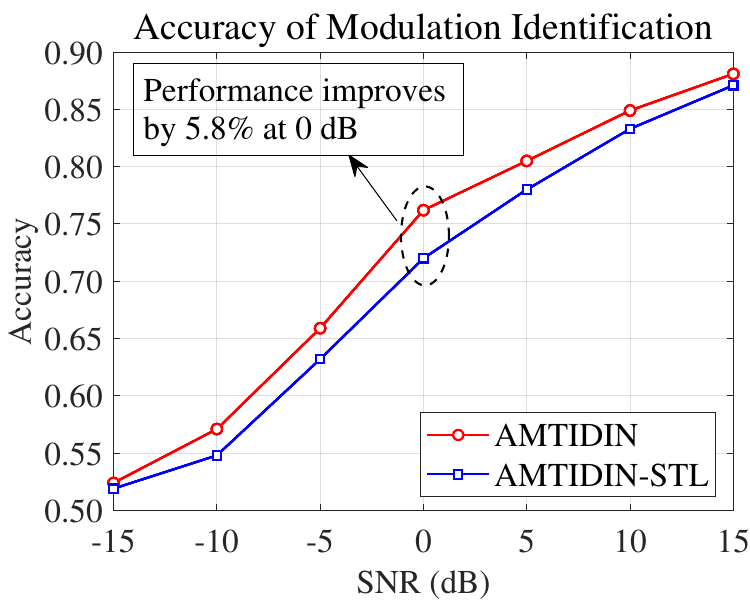}%
\label{fig-M_S_SNR_MI}}
\hfil
\subfloat[Accuracy of interference identification]{\includegraphics[width=0.32\linewidth]{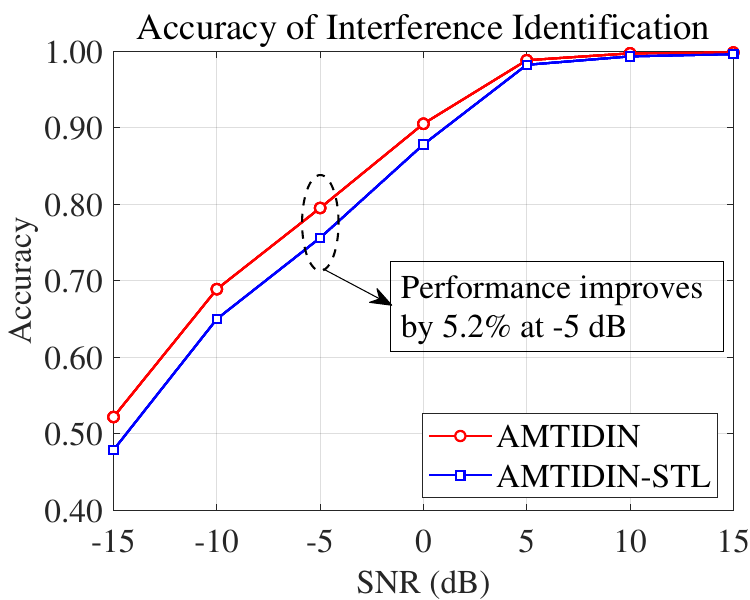}%
\label{fig-M_S_SNR_II}}
\caption{Classification accuracy comparison between AMTIDIN and AMTIDIN-STL under varying SNR levels.}
% (a) The classification accuracy of ID. (b) The classification accuracy of MI. (c) The classification accuracy of ID.}
\label{fig-M_S_SNR}
\end{figure*}

\begin{figure*}[!b]
\centering
\subfloat[Accuracy of interference detection \\ (SNR = -13 dB)]{\includegraphics[width=0.32\linewidth]{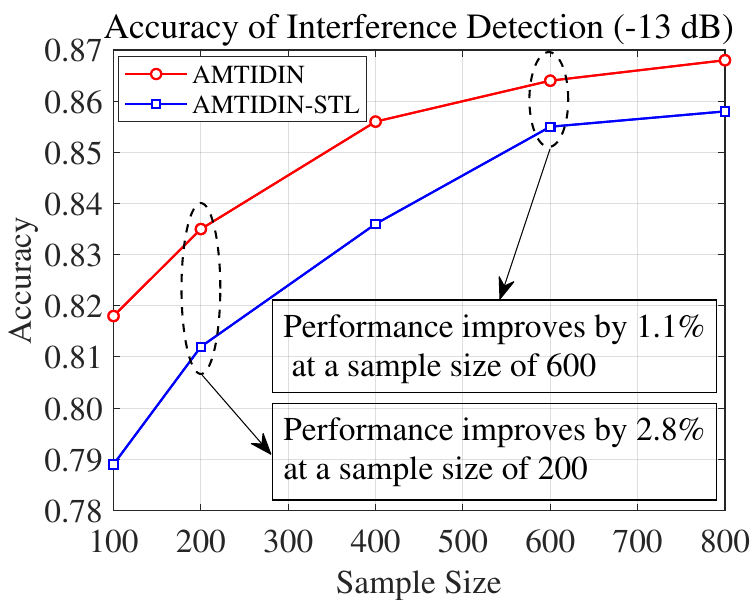}
\label{fig-M_S_SamSize_ID}}
\hfil
\subfloat[Accuracy of modulation identification \\ (SNR = 0 dB)]{\includegraphics[width=0.32\linewidth]{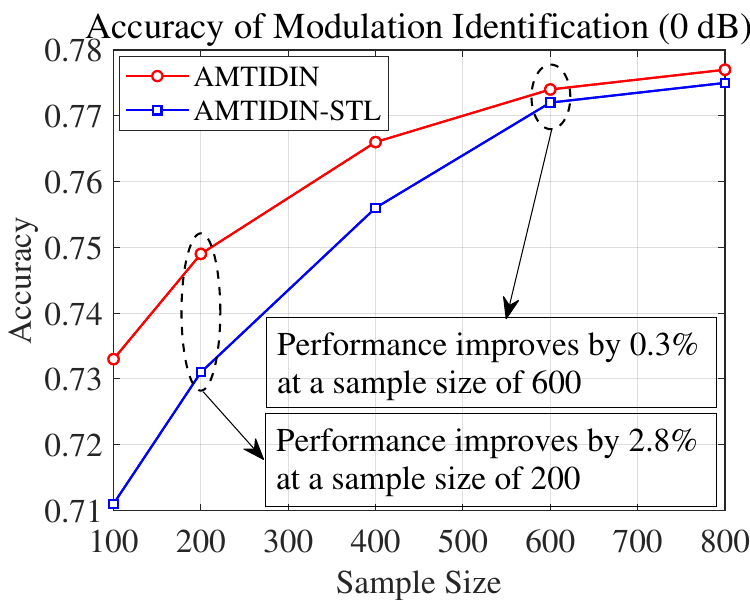}
\label{fig-M_S_SamSize_MI}}
\hfil
\subfloat[Accuracy of interference identification \\ (SNR = -5 dB)]{\includegraphics[width=0.32\linewidth]{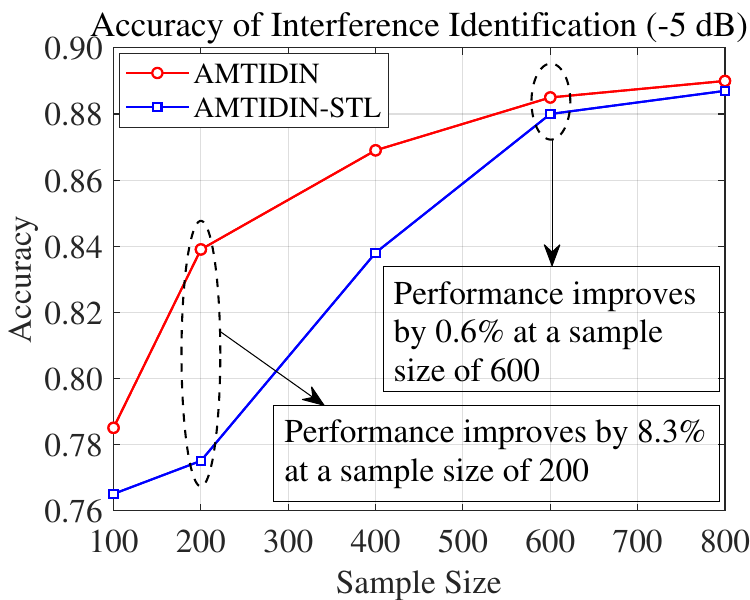}
\label{fig-M_S_SamSize_II}}
\caption{Classification accuracy comparison between AMTIDIN and AMTIDIN-STL under varying sample sizes.}
% (a) The classification accuracy of ID. (b) The classification accuracy of MI. (c) The classification accuracy of II.}
\label{fig-M_S_SamSize}
\end{figure*}

\subsubsection{Robustness against Noise} % Varying SNR Levels
We first analyze performance across varying SNR levels. 
For the data-limited setting, we fix the signal length at 256 and use 100 training samples per class. 
To explicitly quantify performance gains, we focus on the transition regions of the accuracy curves (i.e., the rising slope prior to saturation), where model robustness is most discernible.
Therefore, we select the center points of these transition regions as representative scenarios for analysis.
Based on the results shown in Fig. \ref{fig-M_S_SNR}, we identify -13 dB, 0 dB, and -5 dB as the critical SNR points for interference detection, modulation identification, and interference identification, respectively.
At these specific SNR levels, AMTIDIN demonstrates a marked advantage over its task-specific STL baseline. 
Specifically, for interference detection depicted in Fig. \ref{fig-M_S_SNR}\subref{fig-M_S_SNR_ID}, the proposed method achieves a 4.1\% gain over the baseline at -13 dB, subsequently reaching 92.0\% accuracy at -11 dB. 
For modulation identification shown in Fig. \ref{fig-M_S_SNR}\subref{fig-M_S_SNR_MI}, it outperforms the baseline by 5.8\% at 0 dB. 
Similarly, for interference identification illustrated in Fig. \ref{fig-M_S_SNR}\subref{fig-M_S_SNR_II}, it surpasses the baseline by 5.2\% at -5 dB. 
These results highlight that the MTL framework effectively aggregates shared knowledge to enhance robustness in low-SNR conditions. 
Consequently, these three levels, representing the most discriminative SNR environments, are fixed for subsequent evaluations of sample size and signal length.

\begin{figure*}[!t]
\vspace{-1em}
\centering
\subfloat[Accuracy of interference detection \\ (SNR = -13 dB)]{\includegraphics[width=0.32\linewidth]{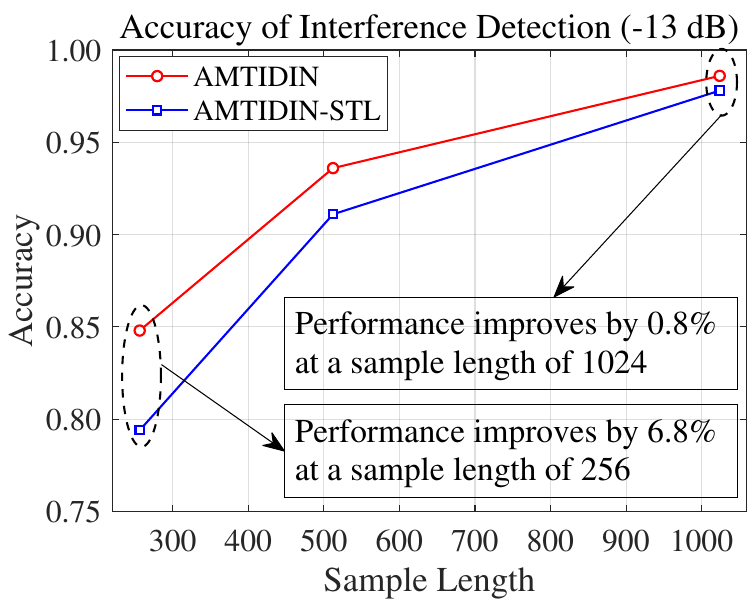}
\label{fig-M_S_SamLen_ID}}
\hfil
\subfloat[Accuracy of modulation identification \\ (SNR = 0 dB)]{\includegraphics[width=0.32\linewidth]{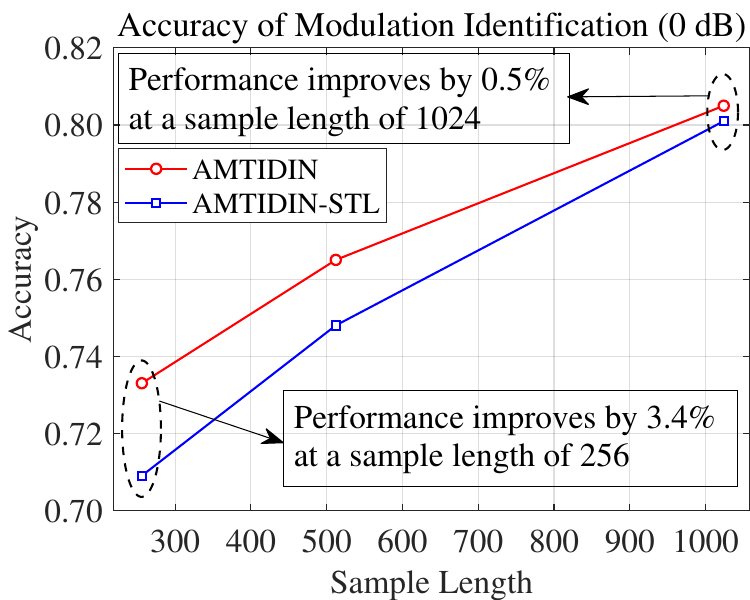}
\label{fig-M_S_SamLen_MI}}
\hfil
\subfloat[Accuracy of interference identification \\ (SNR = -5 dB)]{\includegraphics[width=0.32\linewidth]{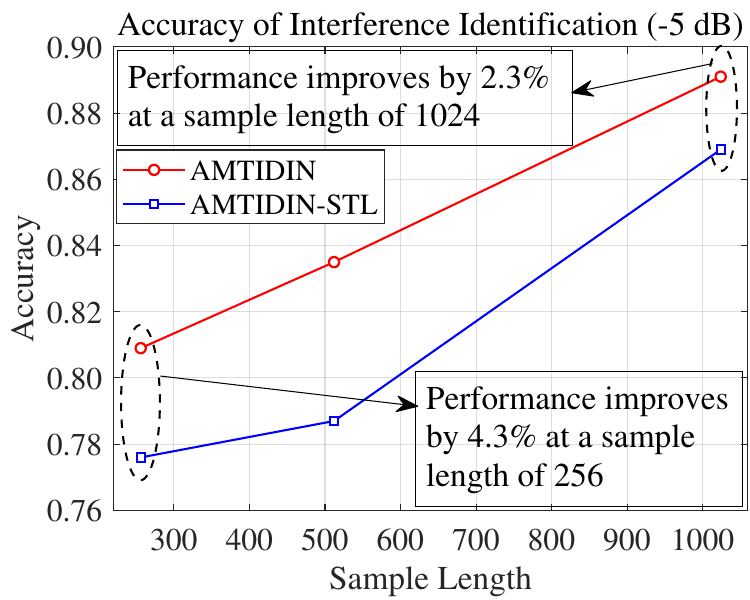}
\label{fig-M_S_SamLen_II}}
\caption{Classification accuracy comparison between AMTIDIN and AMTIDIN-STL under varying sample lengths.} 
% (a) The classification accuracy of ID. (b) The classification accuracy of MI. (c) The classification accuracy of II.}
\label{fig-M_S_SamLen}
\end{figure*}

\subsubsection{Data Efficiency} % Varying Sample Sizes
To assess model robustness under data scarcity, we fix the SNR levels at the aforementioned points (-13 dB, 0 dB, -5 dB) and vary the training sample size from 100 to 800. 
As illustrated in Fig. \ref{fig-M_S_SamSize}, while both models improve with more data, AMTIDIN significantly surpasses the STL counterpart when data is limited. 
At a sample size of 200, AMTIDIN achieves performance improvements of 2.8\%, 2.8\%, and 8.3\% for interference detection, modulation identification, and interference identification, respectively.
As data availability increases to 600, this gap naturally narrows. 
\revised{These results demonstrate improved test accuracy with limited training samples, without increasing model depth or width relative to AMTIDIN-STL.}
\vspace{-1em}

\subsubsection{Sensitivity to Signal Duration} % Varying Sample Lengths
Finally, we evaluate the impact of signal duration (ranging from 256 to 1024 points) under the same fixed SNR conditions. 
Fig. \ref{fig-M_S_SamLen} demonstrates that AMTIDIN provides substantial benefits for short signals, achieving gains of 6.8\%, 3.4\%, and 4.3\%, respectively, at a length of 256. 
While increasing signal length to 1024 improves STL performance by providing richer temporal features, the multi-task advantage remains evident. 
This suggests that the performance bottleneck in STL lies not only in the network capacity but also in the lack of cross-task feature reinforcement. 
AMTIDIN effectively compensates for the lack of discriminative features in short signals by leveraging shared information, maintaining high accuracy where STL struggles.

\begin{table*}[!b]
\centering
\caption{\revised{Comparison of MTL Formulations. Objectives are expressed in the form of~\eqref{equation-optimization_problem_original}; all sums run over the three tasks.}}
\label{table:baseline_models}
\small
\renewcommand{\arraystretch}{1.10}
\setlength{\tabcolsep}{5pt}
\renewcommand{\tabularxcolumn}[1]{m{#1}}
\begin{tabularx}{\linewidth}{@{}>{\centering\arraybackslash}m{0.16\linewidth} >{\centering\arraybackslash}m{0.18\linewidth} >{\centering\arraybackslash}X@{}}
\toprule
\textbf{Model} & \textbf{Description} & \textbf{Objective} \\
\midrule
\textbf{MTL\_Vanilla}~\cite{Caruana1997Multitask} &
Shared features with task-specific heads. &
$\revised{\textstyle\sum_t\lambda_t\hat{\mathcal L}^{\mathrm{CE}}_t(h_t)}$ \\
\addlinespace[2pt]
\textbf{MTL\_Fixed}~\cite{Liu2017Adversarial} &
Adversarial alignment; fixed $\alpha_{t,i}=1/3$. &
$
\revised{\textstyle\frac13\sum_t\lambda_t\sum_i\hat{\mathcal L}^{\mathrm{CE}}_i(h_t)
+\frac{2K}{3}\sum_t\lambda_t\sum_i W_1(\hat{\mathcal D}^i,\hat{\mathcal D}^t)}
$ \\
\addlinespace[2pt]
\textbf{MTL\_NonAdv}~\cite{Murugesan2016Adaptive} &
Learnable task relations; no adversarial alignment. &
$
\revised{\textstyle\sum_t\lambda_t\hat{\mathcal L}^{\mathrm{CE}}_{\bm\alpha_t}(h_t)
+C_1\sum_t\lambda_t\sqrt{\sum_i\frac{\alpha_{t,i}^2}{\beta_i}}}
$ \\
\addlinespace[2pt]
\textbf{AMTIDIN} (Proposed) &
Adversarial alignment and learnable task relations. &
$\revised{\textstyle\sum_t\lambda_t\hat{\mathcal L}^{\mathrm{CE}}_{\bm\alpha_t}(h_t)
+2K\sum_t\lambda_t\sum_i\alpha_{t,i}W_1(\hat{\mathcal D}^i,\hat{\mathcal D}^t)
+C_1\sum_t\lambda_t\sqrt{\sum_i\frac{\alpha_{t,i}^2}{\beta_i}}}$ \\
\bottomrule
\end{tabularx}
\par\smallskip
\begin{minipage}{\linewidth}\footnotesize
\revised{For MTL\_Fixed, coefficient regularization is constant and omitted.}
\end{minipage}
\end{table*}

\subsection{Comparison with MTL Baselines}
\begin{Ronechanged}
We compare different MTL formulations using a consistent set of network modules.
Table~\ref{table:baseline_models} expresses their objectives in a unified form consistent with~\eqref{equation-optimization_problem_original}, highlighting differences in cross-task loss weighting, distribution alignment, and coefficient regularization.
Corresponding feature extraction, task-specific representation, and classification modules use the same architectures across models.
Each model combines these modules according to its learning objective, with adversarial alignment and task relation learning included where applicable.
MTL\_Vanilla~\cite{Caruana1997Multitask} uses conventional shared-feature learning with task-specific classification heads.
MTL\_Fixed~\cite{Liu2017Adversarial} retains adversarial alignment with fixed task relation coefficients, $\alpha_{t,i}=1/3$, whereas MTL\_NonAdv~\cite{Murugesan2016Adaptive} learns these coefficients without adversarial alignment.
All comparisons use a signal length of 256 and 100 training samples per class, consistent with the STL comparison.
\end{Ronechanged}

\begin{figure*}[!t]
\centering
\subfloat[Accuracy of interference detection]{\includegraphics[width=0.32\linewidth]{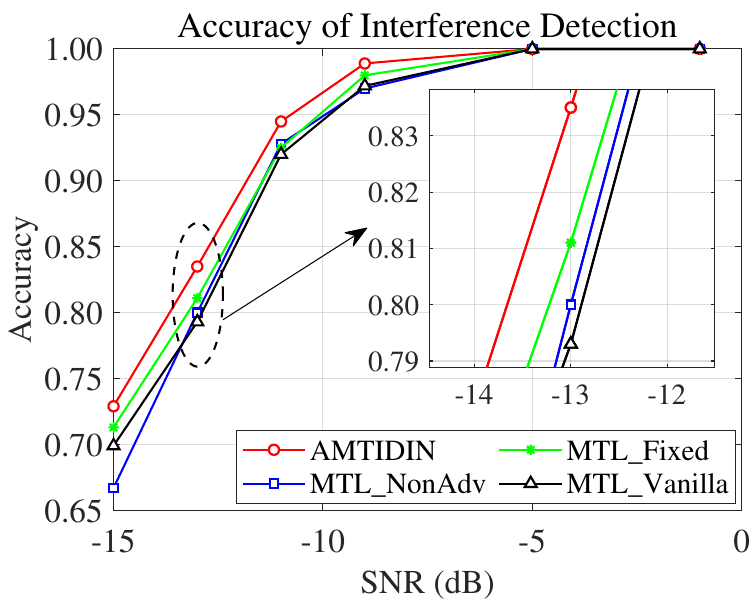}%
\label{fig-M_M_SNR_ID}}
\hfil
\subfloat[Accuracy of modulation identification]{\includegraphics[width=0.32\linewidth]{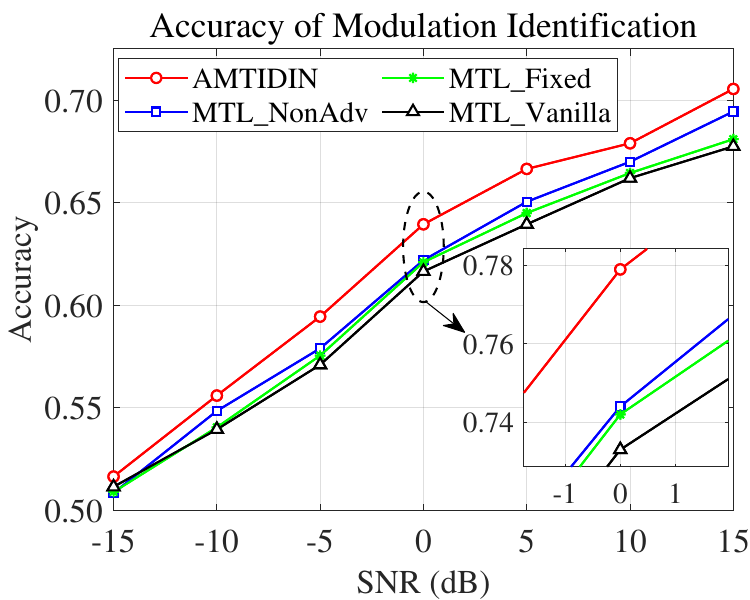}%
\label{fig-M_M_SNR_MI}}
\hfil
\subfloat[Accuracy of interference identification]{\includegraphics[width=0.32\linewidth]{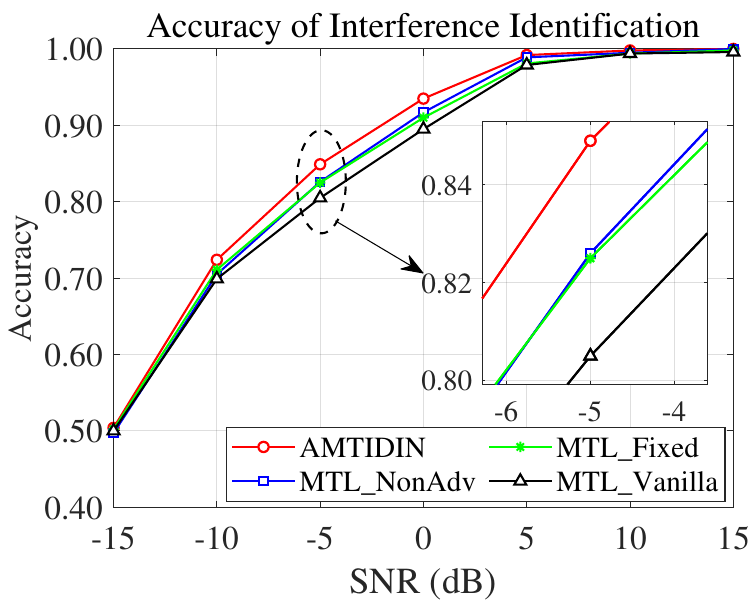}%
\label{fig-M_M_SNR_II}}
\caption{Classification accuracy comparison between AMTIDIN and other MTL models under varying SNR levels.}
%  (a) The classification accuracy of interference detection. (b) The classification accuracy of modulation identification. (c) The classification accuracy of inteference identification.}
\label{fig-M_M_SNR}
\end{figure*}

\begin{figure*}[!t]
\centering
\subfloat[\revised{Cross-entropy loss curves}]{\includegraphics[width=0.32\linewidth]{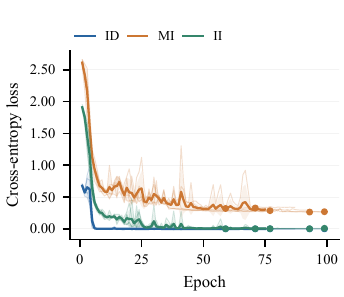}\label{fig:convergence-loss}}
\hfil
\subfloat[\revised{Wasserstein distance curves}]{\includegraphics[width=0.32\linewidth]{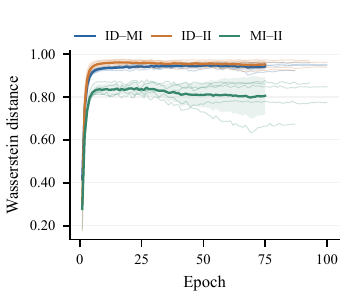}\label{fig:convergence-distance}}
\hfil
\subfloat[\revised{Task relation coefficient curves}]{\includegraphics[width=0.32\linewidth]{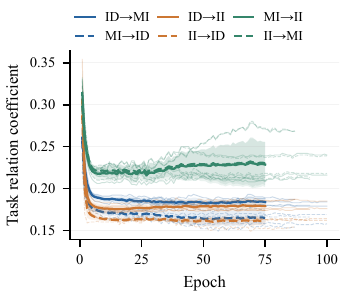}\label{fig:convergence-alpha}}
\caption{\revised{Convergence behavior of AMTIDIN at SNR = 15 dB. Thick curves and shaded bands show the mean and one standard deviation across five runs through their last common epoch. Thin curves retain each run up to its stopping epoch; dots in (a) mark the selected checkpoints. Panel (a) uses validation losses, and panel (b) uses discriminator-based Wasserstein distance estimates. In (c), $t\!\rightarrow\!i$ denotes $\alpha_{t,i}$.}}
\label{fig:convergence}
\end{figure*}

The comparative results for interference detection, modulation identification, and interference identification are illustrated in Fig. \ref{fig-M_M_SNR}. 
\begin{Ronechanged}
The comparisons examine adaptive task weighting, adversarial alignment, and their combined contribution:
\end{Ronechanged}
\begin{itemize}
    \item \textbf{\revised{Adaptive Task Relation Coefficients:}}
    \revised{MTL\_Fixed retains adversarial alignment while fixing $\alpha_{t,i}=1/3$. Its comparison with AMTIDIN isolates adaptive task weighting within the same architecture. As a complementary comparison, MTL\_NonAdv improves over MTL\_Vanilla, particularly for modulation identification in Fig.~\ref{fig-M_M_SNR}\subref{fig-M_M_SNR_MI}, supporting the use of task-weighted cross-task learning.}
    % DRAFT PLACEHOLDER: the MTL_Fixed curve retains original MMoE data pending the new experiment; no Fixed-specific performance claim is added here.
    \item \textbf{\revised{Adversarial Feature Alignment:}}
    \revised{AMTIDIN outperforms MTL\_NonAdv, showing the benefit of adding adversarial alignment to adaptive task weighting. The accuracy margin at $-13$ dB in Fig.~\ref{fig-M_M_SNR}\subref{fig-M_M_SNR_ID} highlights its contribution to interference detection under strong noise.}
    \item \textbf{\revised{Combined Contribution:}}
    \revised{AMTIDIN outperforms MTL\_Vanilla across the three tasks, demonstrating the benefit of combining adversarial alignment with adaptive task weighting beyond standard shared-feature learning. Together, these comparisons support the learning mechanisms motivated by the theoretical bound.}
\end{itemize}
\vspace{-1em}

\subsection{Convergence Analysis}
\label{subsec:convergence}
\begin{Ronechanged}
We examine convergence at 15 dB by tracking task losses, Wasserstein distance estimates, and task relation coefficients during alternating optimization. Experiments use 256-sample inputs and five random seeds, with the training settings specified in Section~\ref{sec:experimental_results}-\ref{subsec:experimental_setup}\ref{subsec:implementation_details}.

As shown in Fig.~\ref{fig:convergence}, the cross-entropy losses decrease in the order of modulation identification, interference identification, and interference detection, consistent with the expected differences in task difficulty. All three sets of curves gradually stabilize in later epochs, demonstrating empirical convergence of the proposed algorithm. The MI--II pair exhibits a smaller Wasserstein distance and larger task relation coefficients than the pairs involving ID. This agrees with the intuition that MI and II share more fine-grained signal characteristics, further supporting the task similarity results in Table~\ref{tab:task_similarity}.
\end{Ronechanged}

\section{Conclusion}
\label{sec:conclusion}
This paper establishes a theoretically grounded adversarial MTL framework for joint interference detection and identification, termed AMTIDIN.
Crucially, we conducted a quantitative analysis of task similarity to offer interpretable insights into intrinsic task relationships. 
This analysis revealed that modulation identification and interference identification share a substantial feature overlap, distinct from the coarse-grained interference detection task. 
\revised{Experiments demonstrate improved classification accuracy over the independently trained single-task counterparts and MTL baselines, particularly with limited training samples, short signals, and low SNRs under the evaluated channel conditions.} 
\revised{These results support the theory-guided combination of adversarial feature alignment and adaptive task weighting for joint interference detection and identification.} 
\revised{Future work will evaluate the framework using over-the-air measurements in dynamic non-cooperative environments, extend it to open-set recognition, and optimize the architecture for real-time edge deployment.}
% \section*{Acknowledgments}
% This should be a simple paragraph before the References to thank those individuals and institutions who have supported your work on this article.

\appendices
\section{Proof of Theorem \ref{theorem-Upper bound}}
\label{appendix:proof_of_theorem_1}

% R1 CHANGE: 07-appendix-multiclass-proof
\begin{Ronechanged}
\subsection{Preliminaries}
Each hypothesis contains one probability output for each task, $h=(h^{(1)},\ldots,h^{(T)})$, with $h^{(i)}(x)\in\Delta_{c_i}$.
Let $P_i$ be the labeled distribution of task $i$, with marginal distribution over inputs $\mathcal D^i$ and labels $y\in\{1,\ldots,c_i\}$.
Using $d_i(\bm p,\bm q)=\tfrac12\|\bm p-\bm q\|_1$, define
\begin{align}
\mathcal L_i(h)&=\mathbb E_{(x,y)\sim P_i}d_i(h^{(i)}(x),\bm e_y),\\
\mathcal L_i(h,h')&=\mathbb E_{x\sim\mathcal D^i}q_i^{h,h'}(x),\\
\label{eq:joint_minimal_expected_loss}
\xi_{i,t}&=\inf_{h\in\mathcal H}[\mathcal L_i(h)+\mathcal L_t(h)].
\end{align}
The empirical and task-weighted risks use the same loss as~\eqref{eq:multiclass_loss}.
Let $S_i=\{(x_j^i,y_j^i)\}_{j=1}^{m_i}\sim P_i^{m_i}$, and let $\sigma_j$ be independent signs, independent of $S_i$, taking values $-1$ and $+1$ with equal probability. The expected Rademacher complexity of $\mathcal F_i$ is
\begin{equation}
\label{eq:rademacher_definition}
\mathfrak R_{m_i}(\mathcal F_i)=\mathbb E_{S_i,\sigma}
\sup_{f\in\mathcal F_i}\frac1{m_i}\sum_{j=1}^{m_i}\sigma_j f(x_j^i,y_j^i).
\end{equation}
Thus, boundedness is imposed on the loss class, without restricting predictions to scalar outputs.

\subsection{Lemmas}
\begin{lemma}[Triangle Inequality for Multiclass Loss]
\label{lem:triangle_inequality_for_loss}
For any $h,h'\in\mathcal H$ and task $i$,
\begin{equation}
|\mathcal L_i(h)-\mathcal L_i(h,h')|\leq\mathcal L_i(h').
\end{equation}
\end{lemma}
\begin{proof}
The reverse triangle inequality for $d_i$ gives
\begin{align*}
&|d_i(h^{(i)}(x),\bm e_y)-d_i(h^{(i)}(x),h'^{(i)}(x))|\\
&\qquad\leq d_i(h'^{(i)}(x),\bm e_y).
\end{align*}
Taking expectations over $(x,y)\sim P_i$ proves the claim for any number of classes.
\end{proof}

\begin{lemma}[Transfer Bound with Task-Specific Outputs]
\label{lem:transfer_bound}
If every $h^{(i)}$ is $K$-Lipschitz under $d_i$, then
\begin{equation}
|\mathcal L_i(h,h')-\mathcal L_t(h,h')|
\leq2K W_1(\mathcal D^i,\mathcal D^t)+\eta_{i,t}.
\end{equation}
\end{lemma}
\begin{proof}
The triangle inequality implies
\[
|q_i^{h,h'}(x)-q_i^{h,h'}(x')|\leq2K\,d_{\mathcal X}(x,x').
\]
Insert $\mathbb E_{\mathcal D^t}q_i^{h,h'}$ between the two risks.
Kantorovich--Rubinstein duality bounds the resulting distribution term by $2K W_1(\mathcal D^i,\mathcal D^t)$.
Definition~\eqref{eq:head_discrepancy} bounds the remaining difference between $q_i^{h,h'}$ and $q_t^{h,h'}$ by $\eta_{i,t}$.
\end{proof}

\begin{lemma}[Wasserstein Concentration]
\label{lem:wasserstein_concentration}
Assume the common metric space has diameter at most one and, for each task, $\mathbb E W_1(\mathcal D^i,\hat{\mathcal D}^i)\leq A_i m_i^{-1/s}$.
Such rates hold under suitable Wasserstein-dimension conditions~\cite{Weed2019Sharp}.
For any $u\in(0,1)$, with probability at least $1-u$,
\begin{equation}
W_1(\mathcal D^i,\hat{\mathcal D}^i)
\leq A_i m_i^{-1/s}+\sqrt{\frac{\log(1/u)}{2m_i}}.
\end{equation}
\end{lemma}
\begin{proof}
Replacing one sample changes the empirical Wasserstein distance by at most $1/m_i$.
McDiarmid's inequality applied to its assumed expectation bound gives the result.
For diameter $D$, the concentration term is multiplied by $D$.
\end{proof}

\begin{lemma}[Uniform Bound for Multiclass Loss Classes]
\label{lem:empirical_loss_bound}
Let $\mathcal Y_i=\{1,\ldots,c_i\}$ and $\mathcal L_{\bm\alpha}(h)=\sum_i\alpha_i\mathcal L_i(h)$. For each $\mathcal F_i\subset[0,1]^{\mathcal X\times\mathcal Y_i}$ and any $u\in(0,1)$, define
\[
\varepsilon_i(u)=2\mathfrak R_{m_i}(\mathcal F_i)
+\sqrt{\frac{\log(T/u)}{2m_i}}.
\]
With probability at least $1-u$, simultaneously for all $h\in\mathcal H$ and $\bm\alpha\in\Delta_T$,
\begin{align}
\mathcal L_{\bm\alpha}(h)
&\leq\hat{\mathcal L}_{\bm\alpha}(h)+\sum_i\alpha_i\varepsilon_i(u)\notag\\
&\leq\hat{\mathcal L}_{\bm\alpha}(h)
+\sqrt{\sum_i\beta_i\varepsilon_i(u)^2}
\sqrt{\sum_i\frac{\alpha_i^2}{\beta_i}}.
\end{align}
\end{lemma}
\begin{proof}
For each task, symmetrization bounds the expected supremum of $\mathcal L_i(h)-\hat{\mathcal L}_i(h)$ by $2\mathfrak R_{m_i}(\mathcal F_i)$~\cite{Bartlett2002Rademacher}.
Changing one observation changes this supremum by at most $1/m_i$.
McDiarmid's inequality and a union bound over the $T$ tasks yield $\mathcal L_i(h)\leq\hat{\mathcal L}_i(h)+\varepsilon_i(u)$ simultaneously for all $i,h$.
Multiply by $\alpha_i$, sum over tasks, and apply Cauchy--Schwarz to $\sum_i(\alpha_i/\sqrt{\beta_i})(\sqrt{\beta_i}\varepsilon_i(u))$.
The event is uniform over $\bm\alpha$, so it also covers coefficients selected from the training samples.
\end{proof}
\vspace{-1em}

\subsection{Proof of Theorem \ref{theorem-Upper bound}}
\begin{proof}
The argument retains three steps: risk transfer, distribution estimation, and empirical-risk control.
\vspace{-1em}

\subsubsection{Step 1: Risk transfer}
For any reference hypothesis $h'\in\mathcal H$, the first two lemmas give
\begin{align}
\label{eq:step1_bounds}
|\mathcal L_i(h)-\mathcal L_t(h)|
&\leq\mathcal L_i(h')+\mathcal L_t(h')\notag\\
&\quad+2K W_1(\mathcal D^i,\mathcal D^t)+\eta_{i,t}.
\end{align}
Taking the infimum over $h'$ introduces $\xi_{i,t}$ without assuming that a minimizer exists.
Multiplying by $\alpha_{t,i}$ and summing yields
\begin{align}
\mathcal L_t(h_t)
&\leq\mathcal L_{\bm\alpha_t}(h_t)
+2K\sum_i\alpha_{t,i}W_1(\mathcal D^i,\mathcal D^t)\notag\\
&\quad+\sum_i\alpha_{t,i}(\xi_{i,t}+\eta_{i,t}).
\end{align}

\vspace{-1em}
\subsubsection{Step 2: Distribution estimation}
By the triangle inequality~\cite{Villani2009Optimal},
\begin{align}
\label{eq:step2_triangle}
W_1(\mathcal D^i,\mathcal D^t)
&\leq W_1(\hat{\mathcal D}^i,\hat{\mathcal D}^t)
+W_1(\mathcal D^i,\hat{\mathcal D}^i)\notag\\
&\quad+W_1(\mathcal D^t,\hat{\mathcal D}^t).
\end{align}
Apply Lemma~\ref{lem:wasserstein_concentration} to both deviations for each ordered task pair with failure probability $\delta/(4T^2)$ per deviation.
A union bound gives, with probability at least $1-\delta/2$, simultaneously for all $i,t$,
\begin{equation}
\label{eq:step2_bounds}
W_1(\mathcal D^i,\mathcal D^t)
\leq W_1(\hat{\mathcal D}^i,\hat{\mathcal D}^t)+\gamma_{i,t}.
\end{equation}
The resulting distribution estimation term is $C_2=2K\sum_t\lambda_t\sum_i\alpha_{t,i}\gamma_{i,t}$.

\subsubsection{Step 3: Empirical-risk control}
Set $u=\delta/2$ in Lemma~\ref{lem:empirical_loss_bound}.
With probability at least $1-\delta/2$, simultaneously for all $h_t,\bm\alpha_t$,
\begin{equation}
\label{eq:step3_bounds}
\mathcal L_{\bm\alpha_t}(h_t)
\leq\hat{\mathcal L}_{\bm\alpha_t}(h_t)
+C_1\sqrt{\sum_i\frac{\alpha_{t,i}^2}{\beta_i}}.
\end{equation}
Combine the three steps, multiply by $\lambda_t$, and sum over $t$.
The union of the two failure events has probability at most $\delta$, proving~\eqref{equation-Upper bound}.
\end{proof}
\end{Ronechanged}
\bibliographystyle{IEEEtran}
\begingroup\emergencystretch=1em
\bibliography{refs}
\endgroup

\end{document}